\documentclass[letterpaper]{article} 
\usepackage{aaai24}  
\usepackage{times}  
\usepackage{helvet}  
\usepackage{courier}  
\usepackage[hyphens]{url}  
\usepackage{graphicx} 
\usepackage{natbib}  
\usepackage{caption} 
\usepackage{pifont}
\usepackage{booktabs} 
\usepackage{multirow}
\usepackage[table,xcdraw]{xcolor}
\usepackage{amsmath,amssymb,graphicx,url}
\usepackage[ruled,linesnumbered]{algorithm2e}
\SetKwComment{Comment}{/* }{ */}
\RestyleAlgo{ruled}
\usepackage{amsthm}
\usepackage{thmtools,thm-restate,enumitem,mathabx}
\usepackage{float}
\usepackage{subcaption}
\usepackage{microtype}

\newtheorem{assumption}{Assumption}
\newtheorem{theorem}{Theorem}
\newtheorem{condition}{Condition}

\newtheorem{lemma}{Lemma}

\usepackage{makecell}

\DeclareMathOperator*{\argmax}{arg\,max}

\newcommand{\bP}{\mathbb{P}}
\newcommand{\bE}{\mathbb{E}}
\newcommand{\bI}{\mathbb{I}}
\newcommand{\bR}{\mathbb{R}}

\newcommand{\cA}{\mathcal A}
\newcommand{\cB}{\mathcal B}
\newcommand{\cC}{\mathcal C}

\newcommand{\cE}{\mathcal E}
\newcommand{\cF}{\mathcal F}

\newcommand{\cL}{\mathcal L}
\newcommand{\cM}{\mathcal M}
\newcommand{\cN}{\mathcal N}
\newcommand{\cO}{\mathcal O}

\newcommand{\cQ}{\mathcal Q}

\newcommand{\cS}{\mathcal S}

\newcommand{\cV}{\mathcal V}

\newcommand{\cY}{\mathcal Y}

\allowdisplaybreaks

\usepackage{newfloat}
\usepackage{listings}
\DeclareCaptionStyle{ruled}{labelfont=normalfont,labelsep=colon,strut=off} 
\floatstyle{ruled}
\newfloat{listing}{tb}{lst}{}
\floatname{listing}{Listing}
\title{Safe Reinforcement Learning with Instantaneous Constraints: The Role of Aggressive Exploration}
\author{
    Honghao Wei\textsuperscript{\rm 1},
    Xin Liu\textsuperscript{\rm 2},
    Lei Ying\textsuperscript{\rm 3}
}
\affiliations{
     \textsuperscript{\rm 1}Washington State University\\
      \textsuperscript{\rm 2} ShanghaiTech University\\
    \textsuperscript{\rm 3}University of Michigan, Ann Arbor\\
   honghao.wei@wsu.edu, liuxin7@shanghaitech.edu.cn, leiying@umich.edu
}

\begin{document}

\maketitle 

\begin{abstract}
This paper studies safe Reinforcement Learning (safe RL) with linear function approximation and under hard instantaneous constraints where unsafe actions must be avoided at each step. Existing studies have considered safe RL with hard instantaneous constraints, but their approaches rely on several key assumptions: $(i)$ the RL agent knows a safe action set for {\it every} state or knows a {\it safe graph} in which all the state-action-state triples are safe, and $(ii)$ the constraint/cost functions are {\it linear}. In this paper, we consider safe RL with instantaneous hard constraints without assumption $(i)$ and generalize $(ii)$ to Reproducing Kernel Hilbert Space (RKHS). Our proposed algorithm, LSVI-AE, achieves $\tilde{\cO}(\sqrt{d^3H^4K})$ regret and  $\tilde{\cO}(H \sqrt{dK})$ hard constraint violation when the cost function is linear and $\cO(H\gamma_K \sqrt{K})$ hard constraint violation when the cost function belongs to RKHS. Here $K$ is the learning horizon, $H$ is the length of each episode, and $\gamma_K$ is the information gain w.r.t the kernel used to approximate cost functions. Our results achieve the optimal dependency on the learning horizon $K$, matching the lower bound we provide in this paper and demonstrating the efficiency of LSVI-AE. Notably, the design of our approach encourages aggressive policy exploration, providing a unique perspective on safe RL with general cost functions and no prior knowledge of safe actions,  which may be of independent interest.
\end{abstract}

\section{Introduction}
Reinforcement Learning (RL) has shown significant empirical success in improving online decision-making in various applications, including games \citep{SilSchJul_17}, robotic control \cite{AndMarBow_20}, etc. However, in many real-world scenarios, it is essential to consider more than just maximizing rewards. Safety, ethical considerations, and adherence to predefined constraints are crucial aspects, particularly in critical domains like robotics, finance, and healthcare.

RL with instantaneous constraints addresses this need by introducing constraints that the agent must adhere to at every single time step during the learning process. Unlike constraints imposed on the entire trajectory or episode \citep{WeiLiuYin_22-2,WeiLiuYin_22,GhoZhoShr_22,DinWeiYan_20,LiuZhoKal_21,BurHasKal_21,WeiGhoShr_23,SinGupShr_20,DinWeiYan_20,CheJaiLuo_22,EfrManPir_20}, instantaneous constraints demand strict compliance with specified limitations at each moment of decision-making so that unsafe actions should be avoided at each step. For instance, in autonomous vehicles, RL agents must consistently adhere to traffic rules and avoid dangerous maneuvers in {\it any time} to ensure safety. In healthcare, RL algorithms that respect privacy and confidentiality restrictions can recommend personalized treatment plans without violating patient data protection regulations. By enforcing instantaneous hard constraints, RL agents can be trusted and relied upon to operate responsibly in complex and dynamic environments while avoiding unnecessary exploratory actions and adhering to safety guidelines.

Existing literature on safe RL with hard instantaneous constraints has explored various aspects of this complex problem. \cite{AmaAliThr_19,PacGhaBar_21} studied the linear bandit problem with instantaneous constraints, which was extended to safe linear MDP with instantaneous constraints in \cite{AmaThrYan_21}. The most recent work \cite{ShiLiaShr_23} studied designing a safe policy for both unsafe states and actions. However, it is important to note that all of the existing works have made restrictive assumptions. \cite{AmaThrYan_21} requires the knowledge of a safe action for every state, while \cite{ShiLiaShr_23} relies on a known safe subgraph where all state-action-state transitions are guaranteed to be safe. Additionally, all of these approaches require the cost function to have a linear structure, which imposes practical limitations on its applicability. In light of these assumptions, their approaches can ensure safe learning during the entire learning process with high probability due to the inherent capability of the algorithm to construct confidence sets along the state feature vector associated with the known safe actions. This, in turn, engenders a more conservative exploration strategy for the agent. A comparison of the theoretical results and the basic assumptions between our paper and the existing results can be found in Table \ref{tab:results}.

\begin{table*}[!ht]
	\begin{center} 
		\begin{tabular}{|l|c|c|c|}
			\toprule
			 {\bf Algorithm} & {\bf Regret} &  {\bf Cost Function} &{\bf  Assumptions} \\
			\hline
			 LSVI-NEW  & \multirow{2}*{$\tilde{\mathcal{O}}(\frac{dH^3\sqrt{dK}}{\Delta_c} )$}  & 	\multirow{2}*{linear}  &  { {known safe subgraph, star convex sets}}  \\
       \cite{ShiLiaShr_23} & & & Lipschitz rewards/transitions  \\
			\hline
			 SLUCB-QVI  &	\multirow{2}*{$\tilde{\mathcal{O}}(\sqrt{d^3H^4K})$} & 	\multirow{2}*{linear} &  {known safe action for each state}  \\
    \cite{AmaThrYan_21} &  & & star convex sets \\
            \hline
			 {\bf This Paper (LSVI-AE)} &	{$\tilde{\mathcal{O}}(\sqrt{d^3H^4K})$}   &  linear / RKHS & \ding{55} \\
			\bottomrule
		\end{tabular}
	\end{center}
	\caption{Regret on linear MDPs for safe RL with instantaneous hard constraints. Here $d$ is the dimension of the feature mapping, $H$ is the duration of each episode, $K$ is the total number of episodes, and $\Delta_c$ is a safety-related parameter.}\label{tab:results}
\end{table*}

In this paper, we study safe RL with minimal assumptions, where we neither assume any prior knowledge of cost/constraint functions nor any form of safe guidance (e.g., safe actions or graphs), except the necessary assumption that the cost functions are within Reproducing Kernel Hilbert Space (RKHS) to guarantee the learnability of cost functions. Since the agent requires to explore the environment from scratch, the constraint violation is {\it unavoidable}. We consider the strict  hard constraint violation, defined as $\sum_{k=1}^K\sum_{h=1}^H g_h(x_h^k,a_h^k)_+$, which prohibits the cancellation across different steps. Here, $K$ represents the total number of episodes, $H$ is the horizon of the (MDP), $g_h$ is the cost function at step $h$, $(x_h^k,a_h^k)$ denotes the state-action pair selected at step $h$ during episode $k$, and $g_h(\cdot)_+:=\max\{g_h(\cdot),0\}$. The hard constraint violation is much more challenging to minimize than the ``soft" constraint violation $[\sum_{k=1}^K\sum_{h=1}^H g_h(x_h^k,a_h^k)]_+$. For example, if we consider a sequence of decisions such that $g_h(x_h^k,a_h^k)=-1$ if $k$ is odd and $+1$ if $k$ is even. Any positive value of the cost indicates a violation of the constraint. Then, assuming $K=100$, it becomes evident that the soft violation is $0$, but the constraint actually violates half of the $K$ episodes. Therefore, an agent/policy with minimal hard violations can guarantee strong safety. Our main contributions of this paper are summarized below:

\begin{itemize}[leftmargin=*]
\item We propose a novel algorithm, LSVI-AE, an acronym for {\bf L}east-squares {\bf V}alue {\bf I}teration with {\bf A}ggressive {\bf E}xploration, which integrates adaptive penalty-based optimization with double optimistic learning. The algorithm guarantees fast learning in an uncertain environment while keeping the hard violation minimal (safe and aggressive exploration). Our design is based on the intuition that aggressive exploration in the initial periods can significantly improve safety and efficiency for the majority of subsequent periods, which is in contrast to the conventional idea of conservative exploration, typically employed in the previous study of safe bandits or RL.

\item We prove that LSVI-AE achieves a regret of $\cO(\sqrt{d^3H^4K})$, and a hard constraint violation of $\cO(H\gamma_K\sqrt{K})$ (the violation becomes $\tilde{\cO}(H\sqrt{dK})$ when the cost functions are linear). To show the sharpness of these results, we provide lower bounds on regret, which is $\Omega(Hd\sqrt{HK})$, and on the violation which is $\Omega(\sqrt{HK})$.  The lower bounds show that LSVI-AE achieves the order-optimal regret and violation w.r.t. the episode length $K$, while the dependencies on $d$ and $H$ can be further improved to match the lower bound using the technique of the ``rare-switching'' idea \citep{HuCheHua_22,HeZhaZHo_22}. To the best of our knowledge, these are the first results in safe RL with instantaneous hard constraints. Further, the numerical experiments verify the ``safe learning'' of our algorithm. 
\end{itemize}

\subsection{Related Work}
 Safe RL, especially those with expected cumulative constraints, has been extensively studied under model-free approaches \citep{WeiLiuYin_22,WeiLiuYin_22-2,WeiGhoShr_23,GhoZhoShr_22}, and model-based approaches\citep{DinWeiYan_20,LiuZhoKal_21,BurHasKal_21,SinGupShr_20,DinWeiYan_20,CheJaiLuo_22}. There are also many works \citep{LiuJiaLi_22,WuCheYan_18,cARDimMor_14} that have studied the knapsack constraints, wherein the learning process stops whenever the budget has run out. \cite{AmaAliThr_19,PacGhaBar_21} studied safe linear bandits which require a linear safety value for each step to be bounded. \cite{ThuBerFel_17,WacSuiYue_18} investigated instantaneous hard constraints with unsafe states under deterministic transitions. \cite{AmaThrYan_21,ShiLiaShr_23} studied linear MDPs with instantaneous hard constraints but with known safe actions or a safe subgraph, and only for the case with linear cost functions.
\section{Problem Formulation}
We consider an episodic Markov decision process (MDP) denoted by $M=(\cS,\cA, H,\bP,r,g),$ where $\cS$ is the state set, $\cA$ is the action set, $H$ is the length of each episode, $\bP=\{\bP_h \}_{h=1}^H$ are the transition kernels at step $h,$ $r=\{r_h\}_{h=1}^H$ are the reward functions, and $g=\{r_h\}_{h=1}^H$ are the cost functions. We assume that $\cS$ is a measurable space with possibly infinite number of elements, $\cA$ is a finite action set. For any $h\in[H],$ the reward function $r_h:\cS\times\cA\rightarrow [0,1],$ is assumed to be deterministic. However, it can be readily extended to settings where $r_h$ is random. The unknown safety measures for taking an action $a$ at state $x$ is a random variable $G_h(x,a)$ with expectation $\bE[G_h(x,a)]=g_h(x,a).$ Without loss of generality, we assume $g_h(x,a):\cS\times\cA\rightarrow [-1,1].$ 

A policy $\pi =\{\pi_h\}_{h=1}^H$ for an agent is a set of functions with $\pi_h:\cS \rightarrow \cA.$ In an episodic MDP, every episode starts by arbitrarily selecting an initial state $x_1$.
In each subsequent step, an agent observes the state $x_h\in\cS,$ takes an action $a_h\in\cA$ according to policy $\pi_h,$ and receives a reward $r_h(x_h,a_h)$ and incurs a cost $g_h(x_h,a_h).$ The MDP then moves to the next state $x_{h+1}$ based on the transition kernel $\bP_h(\cdot \vert x_h,a_h).$ The episode ends after the action $a_H$ is taken at the step $H.$

 Given a policy $\pi,$ let $V_h^\pi(x) : \cS\rightarrow \bR $ denote the expected value of the cumulative reward function starting from step $h$ and state $x,$ when the agent selects action using the policy $\pi=\{\pi_h\}_{h=1}^H,$ which is defined as
$$V_h^\pi(x) = \bE \left[\sum_{i=h}^H r_i(x_i,a_i)\vert x_h=x,\pi  \right], \forall x\in\cS,h\in[H],$$
where $\bE$ is taken with respect to the policy $\pi$ and the transition kernels $\bP.$ Accordingly, we also let $Q_{h}^\pi(x,a):\cS\times\cA\rightarrow \bR$ denote the expected value of the cumulative reward starting from step $h$ and the state-action pair $(x,a)$ and follows the policy $\pi$ as 
\begin{align}
     Q_h^\pi(x,a) = &~\bE \left[\sum_{i=h}^H r_i(x_i,a_i)\vert    x_h=x,a_h=a,\pi  \right], \nonumber \\
      & \forall (x,a)\in \cS\times\cA,\forall h\in[H].
\end{align}
To simplify the notation, we define 
\begin{align}
    [\bP_h V_{h+1}](x,a):= \bE_{x'\sim\bP_h(\cdot\vert x,a)} V_{h+1}(x').
\end{align}
Then we can express the Bellman equation for a given policy $\pi$ as follows:
\begin{align}
  & Q_h^\pi(x,a) = (r_h+\bP_{h}V_{h+_1}^\pi)(x,a), \\
  & V_h^\pi(x) = Q_h^\pi(x,\pi_h(x)), \\
  &  V_{H+1}^\pi(x) =  0.
\end{align}
For an episodic MDP with instantaneous hard constraints, the agent needs to learn the optimal policy while satisfying the constraints at each step of any episode by interacting with the environment. 
The objective of the agent is to find a safe and optimal policy to solve the following problem:
\begin{align}
   &\max_\pi~~ V_1^\pi(x_1) \label{eq:obj} \\
  & \text{s.t.} \quad g_h(x_h,\pi(x_h))\leq 0,\forall h\in[H]. \label{eq:cons}
\end{align}

\begin{assumption}(Feasibility)\label{as:fea}
There exists at least a ``safe'' action for each state $x_h\in\cS,\forall h \in [H].$ 
\end{assumption}
Assumption \ref{as:fea} is necessary to ensure the feasibility of the problem. We remark that the safe actions are unknown to the learner. 

Note that given complete knowledge of reward functions $r_{h}$, cost functions $g_{h}$, and the transition kernel $\mathbb P_{h}$, one could use dynamic (constrained) programming to determine the optimal policy $\pi^{*}$ to \eqref{eq:obj}-\eqref{eq:cons} (thought dynamic programming might suffer from high computational overhead). However, this knowledge is not available in advance, and we have to learn this information while interacting with the environment. 

To measure the performance of an agent in an online learning setting, we consider two metrics w.r.t. rewards and constraints. Let a policy selected by the agent at episode $k$ be $\pi^k=\{\pi_{h}^k\}_{h=1}^H.$ We define the performance metrics:
\begin{align}
    \text{Regret}(K) & = \sum_{k=1}^K V_1^{\pi^*}(x_1^k) - V_1^{\pi^k}(x_1^k),\\
    \text{Violation}(K) & = \sum_{k=1}^K \sum_{h=1}^H \left[g_h(x_h^k,a_h^k)\right]_+,
\end{align}
where $[\cdot]_+=\max\{\cdot,0\}.$ The regret is defined as the gap between the total rewards returned by the optimal policy $\pi^*,$ and that obtained by following the agent's policy $\pi^k$ over $K$ episodes. The constraint violation captures the total constraint violation {\bf without} cancellation over all the episodes $K.$ Note that the violation is unavoidable for an online policy because we do not have the knowledge of the environment (e.g., the cost functions $g_h$). Moreover, the ``hard'' violation is much stricter than the ``soft'' violation $[\sum_{k=1}^K \sum_{h=1}^H g_h(x_h^k,a_h^k)]_+,$ which is especially important for safety-critical applications. 

\subsection{Linear Constrained Markov Decision Processes}
In order to handle a large number or even an infinite number of states, we consider the following linear MDPs.
\begin{assumption}\label{as:linear}
The MDP is a linear MDP with feature map $\phi:\cS\times\cA\rightarrow \bR^d,$ if for any $h,$ there exists $d$ unknown measures $\mu_h=\{\mu_h^1,\dots,\mu_h^d \}$ over $\cS$ such that for any $(x,a,x')\in\cS\times\cA\times\cS,$
\begin{align}
    \bP_h(x'\vert x,a) = \langle\phi(x,a),\mu_h(x')\rangle,
\end{align}
and there exists vector $\theta_{r,h} \in\bR^d$ such that for any $(x,a)\in\cS\times\cA,$ $$r_h(x,a)=\langle \phi(x,a),\theta_{r,h}\rangle.$$ 
\end{assumption}

With loss of generality, we assume $\Vert \phi(x,a)\Vert \leq 1,$ for all $(x,a)\in\cS\times\cA,$ and $\max\{\Vert \mu_h(\cS)\Vert,\Vert \theta_{r,h}\Vert \}\leq \sqrt{d}$ for all $h\in[H].$ 

Under Assumption \ref{as:linear}, we know that \citep{JinYanWan_20} for a linear MDP and any policy $\pi,$ there exists $\{w_{h}^\pi\}_{h=1}^H$ such that $$Q_h^\pi(x,a)=\langle w_h^\pi,\phi(x,a)\rangle,\forall (x,a,h)\in \cS\times\cA\times[H].$$

\section{Algorithm}
In this section, we propose our algorithm, called Least-Squares Value Iteration with Aggressive Exploration (LSVI-AE), in Algorithm \ref{alg:hard}. The design of our algorithm is based on an {\it adaptive penalty-based optimization with double optimistic learning framework} to minimize the cumulative hard constraint violation by encouraging aggressive exploration. In this framework, in episode $k,$ at step $h,$ our algorithm learns both the $Q-$value function  $(Q_h^k)$ and the cost function $(\hat{g}_h(x,a))$ optimistically. By imposing an adaptive rectified operator on the estimated cost, actions are selected at each step $h$ to maximize a surrogate function:
\begin{align}\label{eq:choose_action}
    a_h^k = \arg\max_{a} \{ Q_h^k(x_h^k,a)- Z_h^k(\hat{g}_h^k(x_h^k,a)_+)  \}.
\end{align}
The agent's decision-making process encourages aggressive exploration throughout the learning, in contrast to the conservative policies commonly employed in addressing safe RL with episode constraints or budget limitations. This insight highlights a crucial observation: in the context of safe RL with instantaneous hard constraints and no prior knowledge of safe actions, finding a safe policy requires the agent's prompt exploration of actions that might initially appear unsafe. This strategic emphasis on early exploration of potentially risky actions stands as a foundational principle in our approach.

The nonnegative value $Z_h^k$ is an adaptive penalty factor to control cumulative constraint violation. Note that a standard approach to solving an constrained optimization problem is to optimize the Lagrange function instead, that is, to select an action to maximize:
\begin{align}
    L(x_h^k, \nu) := Q_h^k (x_h^k,a) - \nu \hat{g}_h(x_h^k,a), 
\end{align}
where $\nu$ is the dual variable related to the cost $g_h(x,a)\leq 0.$ We approximate the dual variable $\nu$ with an adaptive penalty factor $Z_h^k$ which is updated according to the observed cost function: $ Z_h^{k+1} := Z_h^k + g_h(x_h^k,a_h^k)_+$ to track the constraint violation during learning. The idea behind the adaptive factor $Z_h^k$ lies in two folds. First the operator $\hat{g}_h^k(x,a)_+$ only penalizes the ``unsafe'' actions that do not satisfy the constraints. Secondly, a minimum penalty price $\eta_h^k$ is established as a lower bound for $Z_h^k$ to prevent aggressive decisions when the constraint does not satisfy. Therefore the adapive rectified factor $Z_h^k$ is updated as 
\begin{align}
    Z_h^{k+1} := \max\{Z_h^k + g_h(x_h^k,a_h^k)_+,\eta_h^k\}. \label{eq:zkupdate}
\end{align}
This design is inspired by constrained online convex optimization \cite{GuoLiuWei_22} and constrained bandit optimization \cite{GuoZhuLiu_22}. However, reinforcement learning with instantaneous constraints is much more complicated due to its stateful nature where the states/actions and rewards/costs are all coupled. For example, if a dangerous/unfavorable action has been taken at the initial step in an episode, it might result in cascade effects to the sequential steps. The setting in \cite{GuoLiuWei_22, GuoZhuLiu_22} can be regarded as a special case of $H=1$ in this paper.  

We remark here that another classical method to track constraint violation is using a virtual queue update approach such that the dual variable is updated as 
\begin{align}
    Z_h^{k+1} := \max \{ Z_h^k + g_h(x_h^k,a_h^k), 0\}. \label{eq:zkupdate-primal}
\end{align}
This approach is usually referred to as the primal-dual approach or the drift-plus-penalty method, which is the most commonly used method for dealing with constraint RL/bandits \citep{EfrManPir_20,DinZhaBas_20,DinZhaBas_22,BaiBedAga_22,LiuLiShi_21} or online convex optimization \citep{YiLiYan_22,YiLiYan_21,YuNee_20}. However, this approach or its variants usually require an assumption of Slater's condition or the knowledge of the Slater/slackness constant to achieve a safe policy. The design is primarily due to their target on ``soft violation'', where the virtual queues/dual variables are the proxy for ``soft violation'' and the Slater's condition is to guarantee the bounded violation. Apparently, this design  
 cannot handle the RL setting with instantaneous hard constraints. This observation also has been justified in the simulation results in Section \ref{sec:sim}. 

\begin{algorithm}[!ht]
\caption{Least-Squares Value Iteration with Aggressive Exploration (LSVI-AE) }\label{alg:hard}
{\bf Initialization:} $Z_h^1 = 1, \forall h\in[H],\eta_h^k= k,\forall k\in[K]$ \;
\For{episode $k=1,\dots,K$}{
Receive the initial state $x_1^k=x_1.$\
\For{$h=H,H-1\dots,1$}{
$\Lambda_h^k = \sum_{\tau=1}^{k-1} \phi(x_h^\tau,x_h^\tau)\phi(x_h^\tau,a_h^\tau)^\top +\lambda I$ \;
$w_h^k\leftarrow (\Lambda_h^k)^{-1}[\sum_{\tau=1}^{k-1} \phi(x_h^\tau,a_h^\tau)[r_h(x_h^\tau,a_h^\tau) + V_{h+1}(x_{h+1}^\tau)  ]$ \;
$Q_h^k(\cdot,\cdot)\leftarrow \min\{\langle w_h^k,\phi(\cdot,\cdot)\rangle + \beta(\phi(\cdot,\cdot)^\top(\Lambda_h^k)^{-1}\phi(\cdot,\cdot))^{1/2},H \} $ \;
$a_x = \arg\max_{a} \{Q_h^k(x,a)- Z_h^k(\hat{g}_h^k(x,a  )_+ )\}.$ \\
$V_h^k(x) = Q_h^k(x, a_x ).$  \
}
\For{h=1,\dots,H}{
Take action $a_h^k$ according to Eq.~\eqref{eq:choose_action} and observe the next state $x_{h+1}^k,$ and cost $g_h(x_h^k,a_h^k)$\;
Update estimates of the cost $\hat{g}_h^k(x,a)$ \;
}
\For{h=1,\dots,H}{
$Z_h^{k+1} =  \max\{Z_h^k + ({g}_h(x_h^k,a_h^k))_{+},\eta_h^k\} .$
}
}
\end{algorithm}

Next, we present the idea of double optimism in estimating $Q-$value functions and cost functions $g.$

{\bf Optimistic Estimates of $Q$:} Estimating $Q-$value functions need to solve a regularized least-squares problem \cite{JinYanWan_20}; however, we should use a SARSA-type update instead of $Q-$learning because Bellman optimally is no longer hold in RL with constraints, i.e., the $V_{h+1}(\cdot)$ in Line $8$ in Algorithm \ref{alg:hard} is not a maximize of the $Q_{h+1}$ functions but from the $Q$ function under the current policy.

To encourage exploration, an additional UCB bonus term $\beta(\phi^\top \Lambda_h^{-1}\phi )^{1/2}$ (Line $6$ in Algorithm \ref{alg:hard}) is added when estimating the $Q-$value functions, where $\Lambda_h$ is the Gram matrix of the regularized least-square problem, and $\beta$ is a scalar. The term $(\phi^\top\Lambda_h^{-1}\phi)^{-1}$ basic represents the effective number of samples that the agent has observed so far along the $\phi$ direction, and the bonus term represents the uncertainty along the $\phi$ direction. Therefore, we can prove that the estimate $Q-$value function $Q_h^k$ is always an upper bound of $Q_h^*$ for all state-action pairs (see Lemma \ref{le:over-est}). Proving this property also leverages the design of the adaptive penalty operator on the cost function.

{\bf Optimistic Estimates of $g$:} 
Assuming the cost functions belongs to RKHS, we present the optimistic estimation when $g_h$ are approximated by GP and also illustrate a special case when $g_h$ are approximated by linear functions. 
\begin{itemize}[leftmargin=*]
\item {\bf Gaussian Process approximation of cost functions:} When $G_h(x,a)$ is a Gaussian process. We let $y=(x,a)$ denote a state-action pair and denote $\cY=\cS\times\cA$ to simplify the notation. Gaussian process $GP(\mu(y),ker(y,y'))$ over a state space $y\in\cY$ is specified by its mean $\mu(y)$ and covariance $ker(y,y').$ If we assume that for any $h\in[H]$ the cost function $G_h(y)$ is a Gaussian process such that $g_h(y)=\bE[G_h(y)],$ and $ker_h(y,y')=\bE[(g_h(y)-\mu_h(y))(g_h(y')-\mu_h(y')) ],$ where $ker_h$ is the kernel function associated with the Reproducing Kernel Hilbert Space (RKHS) with a bounded norm. Then given a collection of states and actions $\cB_h^k= \{ y_h^1,\dots,y_h^{k-1} \},$ we use the GP-LCB \cite{ChoGop_17} to optimistically estimate the cost function for $ h\in[H],k\in[K], y\in\cY$ $$\hat{g}_h^k (y) = g_h^k(y)- \beta_h^k(p/H)  \sigma_h^k(y),$$ 
where $\beta_h^k(p) = 1+ \sqrt{2(\gamma_h^k+1+\ln(2/p) )}$ with $p\in (0,1).$ The information gain $\gamma^k_h:=\max_{y\in\cY: }\frac{1}{2}\ln \vert I+\lambda^{-1} KER_h^k \vert .$ The estimate model includes parameters $\{\mu_h^k,\sigma_h^k\}_{h=1}^H$ and for $h\in[H]$ they are updated as:
\begin{align*}
    & g_h^k(y) = ker_h^k(y)(V_h^k(\lambda))^{-1}g_h^{1:k} \\
    & ker_h^k(y,y') =  ker_h(y,y')
     - ker_h^k(y)^\top (V_h^k(\lambda))^{-1}ker_h^k(y') \\
     &\sigma_h^k(y) = \sqrt{ ker_h^k(y,y)},
\end{align*}
where 
$V_h^k(\lambda) = KER_h^k+\lambda I, \lambda = 1+2/K,
    KER_h^k = [ker_h(y,y')]_{y,y'\in\cB_h^k }, 
 g_h^{1:k} =  \{g_h^1(y_h^1),\dots g_h^{k-1}(y_h^{k-1}) \},
$
and 
$ker_h^k(y) = [ker_h(y_h^1,y), \dots,ker_h(y^{k-1}_h,y) ]^\top.$ Without loss of generality, we assume that the RKHS norm of the cost function is bounded, i.e., $\Vert f\Vert_{ker}=\sqrt{\langle f,f\rangle_{ker}}\leq 1.$ 
\item  {\bf Linear function approximation for cost functions:} For any $h\in[H],(x,a)\in\cS\times\cA,$ the cost function $g_h(x,a):\cS\times\cA\rightarrow [-1,1]$ is assumed to be linear such that there exists vector $\theta_{g,h}\in\bR^d$ and $g_h(x,a)= \langle \phi(x,a),\theta_{g,h}\rangle.$ Recall that at the $k$th episode, we have the Gram matrix $\Lambda_h^k = \sum_{\tau=1}^{k-1} \phi(x_h^\tau,a_h^\tau)\phi(x_h^\tau,a_h^\tau)^\top +\lambda I$ and then we can have an optimization for any $(x,a)$ at the step $h$ with high probability according to:
\begin{align*}
\hat\theta_h^k(x,a) =&~ (\Lambda_h^{k})^{-1}\sum_{\tau=1}^{k-1} \phi(x_h^\tau,a_h^\tau)g_h(x_h^\tau,a_h^\tau)\\
\tilde\beta_h^k(p)=&~ \sqrt{\lambda d} + \sqrt{d\log((1+k/\lambda)/p) } \\
  \hat{g}_h^k(x,a) =&~\langle \phi(x,a), \hat{\theta}_h^k(x,a)\rangle - \tilde\beta_h^k(p/H)\Vert \phi(x,a)\Vert_{(\Lambda_h^k)^{-1}},
\end{align*}
where $\Vert x\Vert_{\Sigma}=\sqrt{x^\top\Sigma x}.$
\end{itemize} 
Note that $\hat g_h^k(x,a)$ is called an optimistic estimation of $g_h(x,a)$ because we are optimistic about $g_h(x,a)\leq 0$, which would imply $\hat{g}_h^k(x,a)\leq 0$ with high probability. Next, we introduce an important condition on the estimation error, which is the key to quantify regret and violation. 
\begin{condition}\label{condi}
    There exist nonnegative values $e^k_h(p,x,a),$ we have for all $x\in\cS,a\in\cA,$ and $h\in[H], k\in[K],$ for any $p\in (0,1)$ we have with probability at least $1-p:$
    \begin{equation}\label{eq:condition}
        0\leq g_h(x,a) - \hat{g}_h^k(x,a) \leq e^k_h(p,x,a),
    \end{equation}
    where $e_h^k(p,x,a)=2\tilde\beta_h^k(p/H)\Vert\phi(x,a)\Vert_{(\Lambda_h^k)^{-1}}$ for the linear case and $2\beta_h^k(p/H)\sigma_h^k(x,a)$ for the Gaussian approximation case.
\end{condition}
We will show that Condition \ref{condi} is satisfied by our optimistic learning in Lemma \ref{le:under_est}, and we defer the proof to the appendix due to the page limit.

\section{Main Results}
In this section, we present the main theoretical result of our algorithm (LSVI-AE), which includes a double optimistic estimation and an adaptive penalty-based rectified factor to encourage aggressive exploration. We also present a theorem that establishes an information-theoretic lower bound for episodic MDP with instantaneous hard constraints to show the tightness of our results. 
\subsection{Performance Guarantee}\label{sec:regret}
Our results are shown as follows:
\begin{theorem}\label{the:main}
Under Condition \ref{condi} and Assumptions \ref{as:fea} and \ref{as:linear}, there exists an absolute constant $c>0$ that for any fixed $p\in(0,1/2),$ if we set $\lambda=1,\beta=cd H\sqrt{\iota}$ in Algorithm \ref{alg:hard} with $\iota=\log(2dHK/p),$ then with probability at least $1-2p$, the total regret and violation of Algorithm \ref{alg:hard} satisfy:
 \begin{align*}
   \text{Regret}(K)  = & \sum_{k=1}^K V_1^{\pi^*}(x_1^k) - V_1^{\pi^k}(x_1^k) =  {\cO}(\sqrt{d^3H^4K\iota^2}),  \\
    \text{Violation}(K)  =& \sum_{k=1}^K \sum_{h=1}^H g_h (x_h^k,a_h^k)_+  \\
    \leq  &\sum_{k=1}^K\sum_{h=1}^H e_h^k(p,x,a) + 2H^2\log(K).
    \end{align*}   
\end{theorem}
We can observe that the dominant term for constraint violation comes from the error in the estimation of cost functions. For the different types of cost functions mentioned above, we have the following results. 
\begin{itemize}[leftmargin=*]
  \item {\bf Gaussian Processes:}
    \begin{lemma}\label{le:err-ga}
        Considering the cost function is a Gaussian process, the cumulative estimation error can be bounded as follows:
        \begin{align}
           & \sum_{k=1}^K\sum_{h=1}^H e_h^k(p,x,a) 
            \leq \cO(H \gamma_K \sqrt{K}),
        \end{align} 
        where $\gamma_K= \max_{h}\{ \gamma_h^K\}.$
    \end{lemma}
    \item {\bf Linear cost function:} 
    \begin{lemma}\label{le:err-linear}
       Considering the cost function in a linear structure, the cumulative estimation error can be bounded as:
    \begin{align}
        \sum_{k=1}^K\sum_{h=1}^H e_h^k(p,x,a) \leq \tilde{\cO}(H\sqrt{dK})
    \end{align}
    \end{lemma}
  
\end{itemize}

\subsection{Lower Bound}\label{sec:lower}
To demonstrate the sharpness of our results, We construct a hard-to-learn linear CMDP with the same state space $\cS,$ action space $\cA,$ episode length $H,$ reward function $\{r_h\}_{h=1}^H,$ cost function $\{g_h\}_{h=1}^H$ and the transition kernel $\{\bP_{h}\}_{h=1}^H$ as in \cite{ZhoGuSze_21,HuCheHua_22}. This is the first result in safe RL under instantaneous hard constraints. The information-theoretic lower bound for the episodic CMDP with hard instantaneous constraints setting studied in this paper is shown is the following theorem.
\begin{theorem}\label{the:lower}
    Let $d\geq 4,H\geq 3,$ and suppose that $K\geq \max\{(d-1)^2H/2,(d-1)/(32H(d-1))\}.$ Then there exists an episodic linear CMDP parameterized by $\mu_h,\theta$ and satisfies the norm assumption given in Assumption \ref{as:linear}, such that the expected regret and violation of constraints are lower bounded as follows by using any algorithm:
    \begin{align}
        \bE[\text{Regret}(K)] =&\Omega(Hd\sqrt{HK}), \\
        \bE[\text{Violation}(K)] = &\Omega(\sqrt{HK}).
    \end{align}
\end{theorem}
We can observe that both our regret and violation have the optimal dependencies on the episode length $K$ when $p\leq 1/\sqrt{K}.$ The dependencies on $d$ and $H$ can be further improved to match the lower bound using the technique of the ``rare-switching'' idea in \cite{HuCheHua_22,HeZhaZHo_22}.
\subsection{Discussion on Extension to More General MDPs:} 
We would like to discuss that our adaptive penalty-based optimization with a double optimistic learning framework can be generalized to general function approximation beyond linear MDPs when the cost function belongs to RKHS. The more general LSVI-AE with function approximation is meant to solve a least-squares regression problem: 
\begin{align}
    \hat{Q}_h^k\leftarrow \min_{f\in\cF} \bigg\{ \sum_{\tau=1}^{k-1}[r_h(x_h^\tau,a_h^\tau) + V_{h+1}^k(x_{h+1}^\tau) \nonumber \\
    ~~~~~~~~~~~~~~~~~- f(x_h^\tau,a_h^\tau)   ]^2 + pen(f)\bigg\},
\end{align}
where $pen(f)$ is a regularization term, $\cF$ is a function class. Then to ensure an overestimation, we can update $Q$ function by adding a bonus term $b_h^k:\cS\times\cA\rightarrow \bR:$
\begin{align}
    Q_h^k(x,a) := \min \big\{ \hat Q_h^k+\beta\cdot b_h^k(x,a),H\big\},
\end{align}
and $V_h^k(x) = Q_h^k(x,a),$ where $$ a=\argmax_{a'\in\cA}\{Q_h^k(x,a') - Z_h^k(\hat g_h^k(x_h^k,a'))_+ \}.$$
The function $\cF$ can be chosen as a RKHS as in \cite{YanJinWan_20}, which covers the linear MDP discussed in this paper, or a more general function $\cF$ with low Bellman Eluder Dimension \cite{JinLiuMir_21}. The main contribution of this paper is proposing a framework for dealing with safe RL under instantaneous hard constraints. The framework can naturally be adopted in more advanced settings dealing with least-squares regression problems.  

\subsection{Proof of Theorem \ref{the:main}}
In this section, we briefly review the key intuitions behind the main results in Theorem \ref{the:main}. We first introduce two lemmas that are useful to prove the theorem. The first lemma shows that $Q_h^k$ is always an upper bound on $Q_h^*$ at any episode $k.$

\begin{lemma}\label{le:over-est}
    Given the event $\cE$ defined in Lemma \ref{le:concentration} and condition \ref{condi}, the following inequality holds simultaneously for all $(x,a)$, step $h$ and episode $k,$  
    \begin{align}
       Q_{h}^k(x,a) \geq Q_h^*(x,a).
    \end{align}
\end{lemma}

Next, the lemma below is used to bound the difference between the value function maintained in Algorithm \ref{alg:hard} and the true value function under policy $\pi^k$ used in each episode $k.$ The error can be bounded with high probability.

\begin{lemma}\label{le:vk-vpik}
Under the event $\cE$ defined in Lemma \ref{le:concentration}, for any fixed $p\in(0,1),$ if we set $\lambda=1,\beta=c\cdot dH\sqrt{\iota}$ in Algorithm \ref{alg:hard} with $\iota=\log(2dHK/p),$ then with probability at least $1-p/2,$ we have :
\begin{align}
    \sum_{k=1}^K V_h^k(x_1^k)- V^{\pi^k}_1(x_1^k) = {\cO}(\sqrt{d^3H^4K\iota^2}). 
\end{align}
\end{lemma}

In the next lemma, we show an upper bound on the entire ``regret plus violation'' term over $K$ episodes using the results from Lemma~\ref{le:over-est} and Lemma~\ref{le:vk-vpik}.

\begin{lemma}\label{le:comb-bound}
Under the event $\cE$ defined in Lemma $\ref{le:concentration}$ and condition \ref{condi}, for any fixed $p\in(0,1),$ we set the parameters in our algorithm as indicated in Lemma \ref{le:vk-vpik}, then with probability at least $1-p/2,$ we have:
    \begin{align}
  & \sum_{k=1}^K  V^*_h(x_h^k) - V_h^{\pi^k}(x_h^k) + Z_h^k(\hat{g}_h^k(x_h^k,a_h^k)_+ ) \nonumber \\
   =& {\cO}(\sqrt{d^3H^4K\iota^2}) 
\end{align}
\end{lemma}
\begin{proof}
For any $h\in[H], k\in[K],$ according to the action selection (Eq.\eqref{eq:choose_action}) in our algorithm we have 
\begin{align}
   & Q_h^k(x_h^k,a_h^k) -Z_h^k(\hat{g}_h^k(x_h^k,a_h^k)_+ ) \nonumber\\
     \geq & Q_{h}^k(x_h^k,a_h^*)- Z_h^k(\hat{g}_h^k(x_h^k,a_h^*)_+ )  \nonumber \\
     &+ Q_h^*(x_h^k,a_h^*)- Q_h^*(x_h^k,a_h^*),
\end{align}
where $a^*_h$ is the optimal action selected by the optimal policy $\pi^*.$ Therefore rearranging the equation and subtracting $Q_h^{\pi^k}(x_h^k,a_h^k)$ at both sides we have:
\begin{align}
   &Q_h^*(x_h^k,a_h^*)-Q_h^{\pi^k}(x_h^k,a_h^k) + Z_h^k (\hat{g}_h^k(x_h^k,a_h^k)_+ ) \nonumber\\
   \leq & Q_h^*(x_h^k,a_h^*) - Q_h^k(x_h^k,a_h^*)  \label{eq:self-term1}\\
   &+ Z_h^k (\hat{g}_h^k(x_h^k,a_h^*)_+ )  \label{eq:self-term2} \\
   &+Q_h^k(x_h^k,a_h^k)-Q^{\pi^k}_h(x_h^k,a_h^k). \label{eq:self-term3}
\end{align}   
 Eq.~\eqref{eq:self-term1} is nonpositive due to the overestimation Lemma \ref{le:over-est}. Eq.~\eqref{eq:self-term2} is also nonpositive because the optimistic estimation of the cost function ensures that $Z_h^k (\hat{g}_h^k(x_h^k,a_h^*)_+ ) \leq Z_h^k ({g}_h (x_h^k,a_h^*)_+ ) \leq 0.$ Bounding the last term \eqref{eq:self-term3} with Lemma \ref{le:vk-vpik} we prove the lemma.
\end{proof}
Using the results from Lemma \ref{le:comb-bound}, we are ready to prove  the main results:

{\noindent \bf Regret:} According to the results from Lemma \ref{le:comb-bound}, we have:
    \begin{align}
         \text{Regret}(K) =  &\sum_{k=1}^K V_1^{\pi^*}(x_1^k) - V_1^{\pi^k}(x_1^k)  \nonumber \\
         = & {\cO}(\sqrt{d^3H^4K\iota^2}) -  Z_1^k(\hat{g}_1^k(x_1^k,a_1^*)_+  \nonumber \\
         =  &  {\cO}(\sqrt{d^3H^4K\iota^2}) 
    \end{align}
{\noindent \bf Violation:} Using the intermediate results in Lemma \ref{le:comb-bound} (Eq.~\eqref{eq:self-term1}-\eqref{eq:self-term3}) we have that: 
\begin{align}
	\hat{g}_h^k(x_h^k,a_h^k)_+  \leq & \frac{1}{Z_h^k}  \bigg([ Q_h^k(x_h^k,a_h^k)-Q^{\pi^k}_h(x_h^k,a_h^k)] )\nonumber \\
	& - [Q_h^*(x_h^k,a_h^*) -Q_h^{\pi^k}(x_h^k,a_h^k) ]\bigg) \nonumber\\
	\leq &\frac{1}{{k}} \bigg\vert  [ Q_h^k(x_h^k,a_h^k)-Q^{\pi^k}_h(x_h^k,a_h^k) ]\nonumber \\
	&-[Q_h^*(x_h^k,a_h^*)-Q_h^{\pi^k}(x_h^k,a_h^k) ]\bigg\vert. \label{eq:vio-pro}
\end{align}
The inequality holds because our choice of $Z_h^k$ in our algorithm such that $Z_h^k\geq\eta_k={k}.$ Therefore we have:
\begin{align}
	& \text{Violation}(K) =  \sum_{k=1}^K \sum_{h=1}^H  g_h (x_h^k,a_h^k)_+ \nonumber \\
	= & \sum_{k=1}^K \sum_{h=1}^H  \left(  g_h (x_h^k,a_h^k)  -  \hat{g}_h (x_h^k,a_h^k) + \hat{g}_h (x_h^k,a_h^k)   \right)_+ \nonumber \\
	\leq  & \sum_{k=1}^K \sum_{h=1}^H  \left(  g_h (x_h^k,a_h^k)  -  \hat{g}_h (x_h^k,a_h^k)  \right)_+  \sum_{k=1}^K \sum_{h=1}^H  \hat{g}_h (x_h^k,a_h^k)_+  \nonumber  \\
	\leq &  \sum_{k=1}^K\sum_{h=1}^H e_h^k(p,x,a)   
	+  \sum_{k=1}^K\sum_{h=1}^H \frac{1}{{k}} \bigg\vert  [Q_h^k(x_h^k,a_h^k)-Q^{\pi^k}_h(x_h^k,a_h^k)]  \nonumber \\
	&~~~~~~~~~~~~~~~~~~~~~~-[(Q_h^*(x_h^k,a_h^*)-Q_h^{\pi^k}(x_h^k,a_h^k) ] \bigg\vert \nonumber \\
	\leq &   \sum_{k=1}^K\sum_{h=1}^H e_h^k(p,x,a) +  \sum_{k=1}^K \frac{2H^2}{{k}} \nonumber\\
	\leq &  \sum_{k=1}^K\sum_{h=1}^H e_h^k(p,x,a) + 2H^2\log(K),
\end{align}
where the first inequality holds because of the fact $(a+b)_+\leq a_+ + b_+,$ the second inequality is due to Eq.\eqref{eq:vio-pro}, the third inequality is because of the assumption that reward is bounded by $1,$ and the last inequality is true by using the fact that $\sum_{k=1}^K\frac{1}{{k}}\leq \int_1^K \frac{1}{{k}} dk \leq \log(K).$

\section{Simulation}\label{sec:sim}
In this section, we evaluate the performance of our algorithm in the Frozen Lake environment \cite{AmaThrYan_21}, as illustrated in Figure~\ref{fig:map}. The agent's objective is to navigate a $10\times 10$ grid map to reach a goal while avoiding hazards. At each time step, four actions are available, with a $0.9$ probability of moving in the intended direction, and a $0.05$ probability for each orthogonal direction. For this simulation, we set $H=15$, $K=1000$, and $d=\vert \cS\vert \times\vert \cA\vert$. The feature vector is defined as $\phi(x,a)=e_{x,a}$, where $e_{x,a}$ is a $d$-dimensional vector with the element corresponding to the state-action pair $(x,a)$ set to $1$ and zero for other values. The agent receives a reward of $6$ upon reaching the goal, and $0.01$ otherwise. Taking dangerous actions (hitting the hazards) incurs a cost of $1$, while safe actions result in a cost of $-1$. If the agent reaches the goal, it remains there until the end of the episode.

To highlight the benefits of our algorithm and its aggressive exploration strategy in addressing safe RL with instantaneous hard constraints, we compare our approach against two baselines: $(1)$ classical Least-Squares Value Iteration (LSVI) \cite{JinYanWan_20} without accounting for any constraints during learning; $(2)$ LSVI-Primal, representing the virtual queue (dual variable) update based on Eq. \eqref{eq:zkupdate-primal} in the traditional primal-dual/drift-plus-penalty approach for dealing with long-term or budget constraints in safe RL.

\begin{figure}[!ht]
	\centering
	\includegraphics[height=3.5cm, width=6cm]{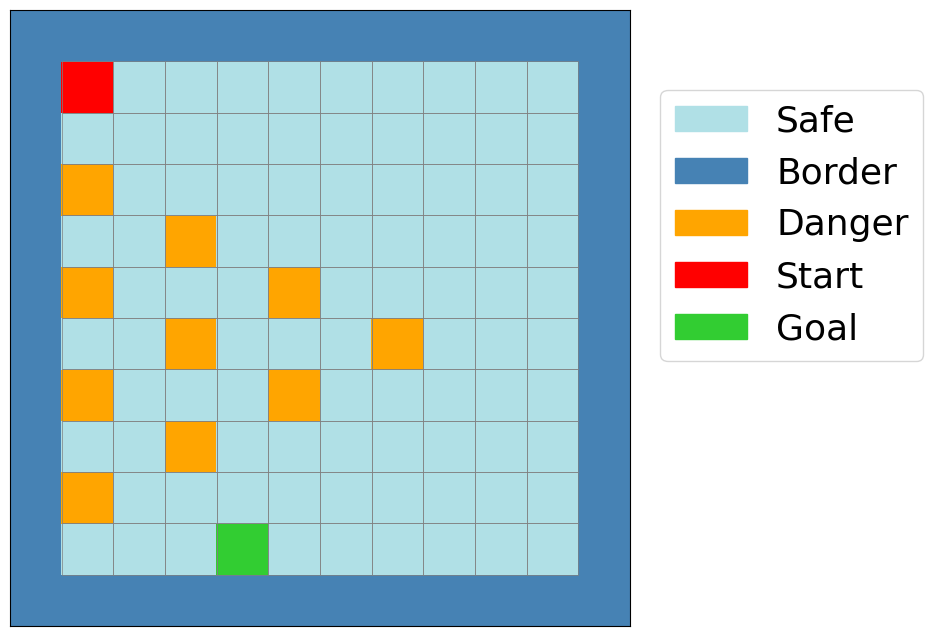}
	\caption{Frozen Lake Environment}
	\label{fig:map}
\end{figure}

We present the results of our evaluation in Figure \ref{fig:result}, depicting the moving average reward and the cost $(\sum_{h=1}^H g_h(x_h,a_h)+)$ return. Our LSVI-AE algorithm obtains an optimal reward comparable to that achieved by the LSVI algorithm designed for unconstrained MDPs. However, our approach significantly outperforms in terms of cost. Intriguingly, the LSVI-Primal approach designed for episodic constraint scenarios fails to perform effectively in this environment, exhibiting limited learning progress and only surpassing the unconstrained case in terms of cost. However, this cost improvement still fails to guarantee the desired performance of safe RL with instantaneous hard constraints, where the objective is to ensure $\sum_{h=1}^H g_h(x_h,a_h)_+ \leq 0$. These observations validate the key principles underlying our approach. More discussions can be found in Section \ref{ap:sim} in the appendix.
\begin{figure}[!ht]
	\centering
	\includegraphics[width=1.0\linewidth]{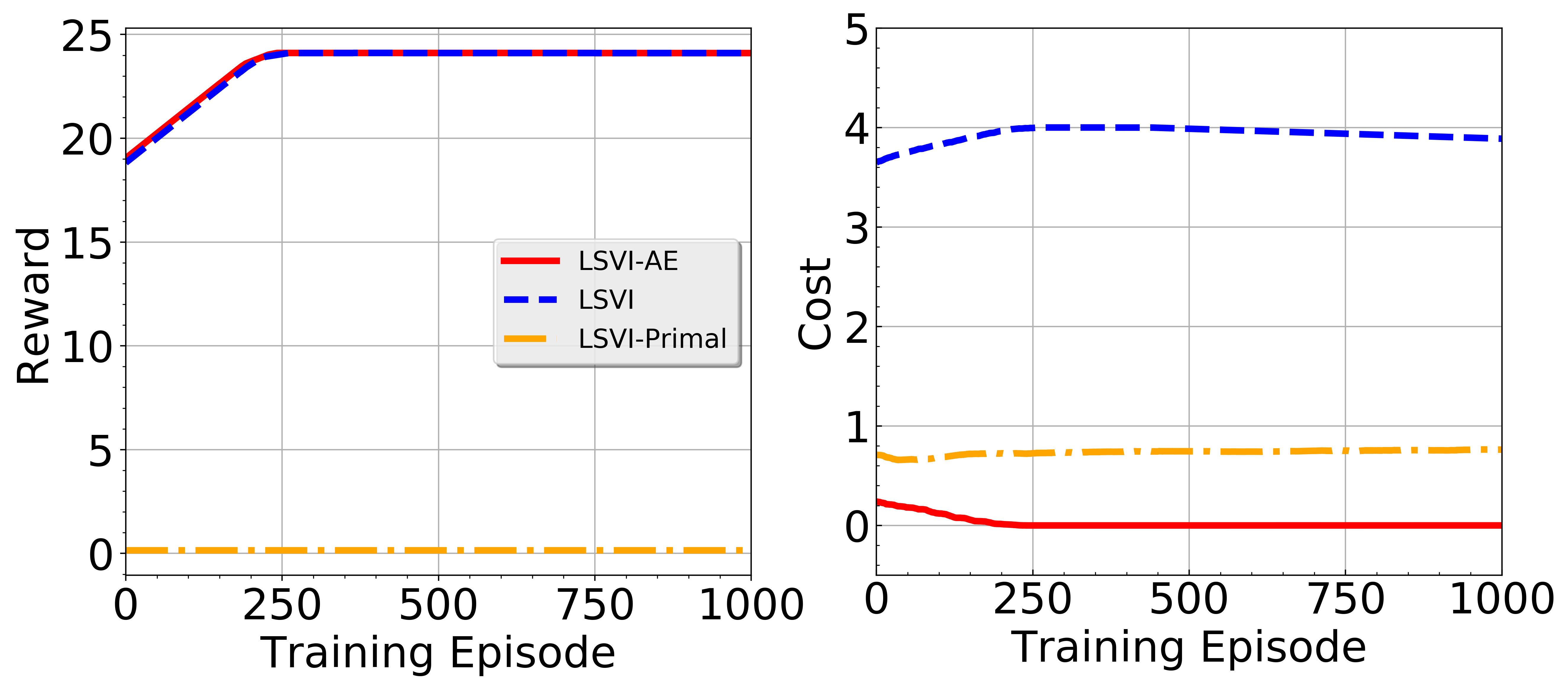}
	\caption{Reward and Cost Performance During Training}
	\label{fig:result}
\end{figure}
\section{Conclusion}
In this paper, we introduce LSVI-AE, an innovative algorithm designed for safe RL with instantaneous constraints, addressing scenarios where no prior knowledge of safe actions or safe graphs is available. For the first time, we propose an adaptive penalty-based optimization with a double optimistic learning framework for taking care of this setting under a more general cost function. Our approach establishes both sub-linear regret bound and hard constraint violation bound, which both are optimal w.r.t $K$ and match the information-theoretic lower bound. A notable feature of our approach lies in its emphasis on promoting aggressive policy exploration, contributing to the paradigm of algorithm design in this context. 


\newpage

\appendix
\onecolumn
\section*{Appendix}

{\bf Notations:} We denote $Q_h^k$ as the $Q-$value function estimate and $\Lambda_h^k,w_h^k$ as the parameters in episode $k.$ The value function $V_h^k$ is denoted as $V_h^k(x)= Q_h^k(x,a_h^k),$ $a_h^k$ is the action chosen at step $h$ in the $k-$th episode. To simplify the presentation, we denote $\phi_h^k :=\phi(x_h^k,a_h^k).$ With loss of generality, we assume that $\Vert \phi_h(x,a)\Vert \leq 1$ for all $(x,a)\in\cS\times \cA, \Vert \mu_h(\cS)\Vert \leq \sqrt{d},$ and $\Vert \theta_{r,h}\Vert \leq \sqrt{d}$ for all $h\in[H].$ 
\section{Auxiliary Lemmas}
We first present the lemma to show the optimistic estimation of $g_h.$
\begin{lemma} \label{le:under_est} For all $(x,a)\in \cS\times\cA,k\in[K], h\in[H]$ considering the following two cases:

\noindent {\bf Linear Case:} 
For the linear case, with probability at least $1-p,$ the estimated cost $\hat{g}_h^k(\cdot,\cdot)$ satisfies that $$0 \leq g_h(x,a) - \hat{g}_h^k(x,a)\leq 2\tilde\beta_h^k(p/H)\Vert\phi(x,a)\Vert_{(\Lambda_h^k)^{-1}}.$$ 
{\bf Gaussian Processes:} For the Gaussian processes, we assume that $ker_h((x,a),(x',a'))\leq 1, \forall (x,a)\in \cS\times\cA),$ then the following inequality holds with probability at least $1-p,$ the estimated cost $\hat{g}_h^k(\cdot,\cdot)$ satisfies that $$0\leq g_h(x,a) - \hat{g}_h^k(x,a)\leq 2\beta_h^k(p/H)\sigma_h^k(x,a)  $$
\end{lemma}

\begin{proof}
First for the linear case, the results come from using the following lemma,
\begin{lemma}[Theorem $2$ in \cite{AbbPalSze_11}]
    For any $h\in[H],p\geq 0,$ with probability at least $1-p,$ the following event occurs 
    $$\Vert (\Lambda_h^k)^{-1}\sum_{\tau=1}^{k-1}\phi(x_h^\tau,g_h^\tau)g_h(x_h^\tau,a_h^\tau) - \theta_{g,h}\Vert_{\Lambda_h^k} \leq \tilde\beta_h^k(p),\forall k \in [K],$$ 
\end{lemma}
where $\Vert x\Vert_{\Sigma}=\sqrt{x^\top\Sigma x}, \Lambda_h^k = \sum_{\tau=1}^{k-1} \phi(x_h^\tau,a_h^\tau)\phi(x_h^\tau,a_h^\tau)^\top +\lambda I,$ and $\tilde\beta_h^k(p)=\sqrt{\lambda d} + \sqrt{d\log((1+k/\lambda)/p) }.$

\noindent Recall that 
\begin{align*}
  \hat{\theta}_k^h =& (\Lambda_h^k)^{-1}\sum_{\tau=1}^{k-1}\phi(x_h^\tau,g_h^\tau)g_h(x_h^\tau,a_h^\tau)\\
    \hat g_h^k(x,a)= &\langle \phi(x,a),  \hat\theta_h^k\rangle - \tilde\beta_h^k(p/H)\Vert\phi(x,a)\Vert_{(\Lambda_h^k)^{-1}}.
\end{align*}
Since we have that
\begin{align*}
 & \vert  \langle \phi(x,a),\theta_{g,h}\rangle - \langle\phi(x,a),\hat{\theta}_h^k\rangle \vert \\
  =& \vert \phi(x,a),  \theta_{g,h} - \hat{\theta}_h^k\rangle \vert \\
  \leq & \Vert \theta_{g,h} - \hat{\theta}_h^k\Vert_{\Lambda_h^k} \Vert\phi(x,a)\Vert_{(\Lambda_h^k)^{-1}} \\
  \leq & \tilde\beta_h^k(p/H)\Vert\phi(x,a)\Vert_{(\Lambda_h^k)^{-1}}
\end{align*}
Using the union bound, it is straightforward to obtain that for all $(x,a)\in \cS\times\cA,k\in[K],h\in[H],$ with probability at least $1-p,$ we have 
$$0\leq g_h(x,a) - \hat{g}_h^k(x,a) \leq 2\tilde\beta_h^k(p/H)\Vert\phi(x,a)\Vert_{(\Lambda_h^k)^{-1}}.$$

\noindent For the Gaussian processes, recall that the cost function is updated as:
$$\hat g_h^k(x,a) = g_h^k(x,a) - \beta_h^k(p/H)\sigma_h^k(x,a), \forall (x,a)\in \cS\times\cA,$$ then we have $$\hat g_h^k(x,a) - g_h(x,a) =  g_h^k(x,a) - g_h(x,a) - \beta_h^k(p/H)\sigma_h^k(x,a).$$ For convenience, we let $y:=(x,a),\cY:=\cS\times\cA.$ We know that $\{g\}_{h=1}^H$ lie in RKHS, then we define $\psi_h(y)=ker_h(y,\cdot),$ it implies that $g_h(y)=\langle g_h,ker_h(y,\cdot)\rangle_{ker_h} = \langle g_h,\psi_h(y)\rangle_{ker_h} := g_h^\top \psi_h(y). $ Further define the RKHS norm $\Vert g_h\Vert_{ker_h}$ as $\sqrt{g_h^\top g_h},$ $\Psi_h^k=[\psi_h(x_1)^\top,\dots,\psi_h(x_{k-1})^\top ]^\top,$ then the kernel matrix is defined as $KER_h^k=\Psi_h^k (\Psi_h^k)^\top, ker_h^k=\Psi_h^k\psi_h(y)
$ for all $y\in\cY,$ and $g_{h}^{1:k} = \Psi_h^k g_h.$ Also we have $g_h^k(y)=ker_h^k(y)^\top (KER_h^k+\lambda I)^{-1} g_h^{1:k} $ Then according to the model update we have
\begin{align}
    g_h^k(y) - g_h(y) = & ker_h^k(y)^\top (KER_h^k+\lambda I )^{-1}g_h^{1:k} - g_h(y).
\end{align}
Using Theorem $2$ in \cite{ChoGop_17}, for any $h$ we can derive the above difference error term as follows 
\begin{align}
   & \vert ker_h^k(y)^\top (KER_h^k+\lambda I )^{-1}g_h^{1:k}  - g_h(y) \vert \nonumber\\
    = & \vert \psi_h(y)^\top (\Psi_h^k)^\top (\Psi_h^k(\Psi_h^k)^\top +\lambda I)^{-1}\Psi_h^k g_h - \psi_h(y)^\top g_h\vert  \nonumber \\
    = & \vert \psi_h(y)^\top( (\Psi_h^k)^\top \Psi_h^k +\lambda I )^{-1}(\Psi_h^k)^\top \Psi_h^k g_h -  \psi_h(y)^\top g_h\vert \nonumber\\
    = & \vert \lambda \psi_h(y)^\top ( (\Psi_h^k)^\top \Psi_h^k +\lambda I )^{-1}g_h\vert \nonumber\\
    \leq & \Vert \lambda \psi_h(y)^\top ( (\Psi_h^k)^\top \Psi_h^k +\lambda I )^{-1}\Vert_{ker_h}\Vert g_h\Vert_{ker_h} \nonumber\\
    = & \Vert g_h\Vert_{ker_h} \sqrt{\lambda \psi_h(y)^\top ((\Psi_h^k)^\top \Psi_h^k +\lambda I )^{-1}(\lambda I )((\Psi_h^k)^\top \Psi_h^k +\lambda I )^{-1} \psi_h(y) } \nonumber\\
    \leq & \sqrt{\lambda \psi_h(y)^\top ((\Psi_h^k)^\top \Psi_h^k +\lambda I )^{-1}((\Psi_h^k)^\top \Psi_h^k +\lambda I )((\Psi_h^k)^\top \Psi_h^k +\lambda I )^{-1} \psi_h(y) }\nonumber\\
    = & \sigma_h^k(y),
\end{align}
where the second equality holds because $((\Psi_h^k)^\top ((\Psi_h^k)^\top \Psi_h^k +\lambda I )^{-1} =  ((\Psi_h^k)^\top \Psi_h^k +\lambda I )^{-1}(\Psi_h^k)^\top  $ and the third equality comes from $\phi_h(y) =  ((\Psi_h^k)^\top ((\Psi_h^k)^\top \Psi_h^k +\lambda I )^{-1} ker_h^k(y) + \lambda ((\Psi_h^k)^\top \Psi_h^k +\lambda I )^{-1} \psi_h(y),$ by using the definition $((\Psi_h^k)^\top \Psi_h^k +\lambda I )\psi_h(y) = (\Psi_h^k)^\top ker_h^k(y)+\lambda\psi_h(y) .$ We can further have $$\lambda \psi_h(y)^\top  ((\Psi_h^k)^\top \Psi_h^k +\lambda I )^{-1} \psi_h(y) = ker_h(y,y)-ker_h^k(y)^\top (KER_h^k+\lambda I)^{-1} ker_h^k(y) = (\sigma_h^k(y))^2,$$ which implies the last equality.  

Therefor according to Theorem $1$ in \cite{ChoGop_17}, we have for all $h\in[H],y\in\cY,k\in[K],$ with probability at least $1-p,$ we have
$$ -\beta_h^k(p/H)\sigma_h^k(y) + \beta_h^k(p/H)\sigma_h^k(y)\leq \hat g_h^k(y)-g_h(y)\leq \beta_h^k(p/H)\sigma_h^k(y)+\beta_h^k(p/H)\sigma_h^k(y).$$ We finish the proof.
\end{proof}
\noindent In the following, we will state several lemmas that will be used in our analysis.

\begin{lemma}\label{le:whpi-bound}
    Under Assumption \ref{as:linear}, for any fixed policy $\pi,$ let $w_h^\pi$ be the corresponding weights such that $Q_h^\pi =\langle \phi(x,a), w_h^\pi\rangle,$ then we have for all $h\in[H],$
    $$\Vert w_h^\pi\Vert \leq 2H\sqrt{d}.$$
\end{lemma}
\begin{proof}
    From the linear structure of the $Q-$value function we have for any policy $\pi$
    \begin{align}
     Q_{h}^\pi (x,a)  = & r_h(x,a) + \bP_h V_h^\pi (x,a) \nonumber \\
     = & \langle \phi(x,a), \theta_{r,h} \rangle + \int_{\cS} V_{h+1}^\pi(x')\langle \phi(x,a),d\mu_h(x')\rangle \nonumber \\
     = & \langle \phi (x,a), w_h^\pi\rangle ,
    \end{align}
    where $w_h^\pi = \theta_{r,h} + \int_{\cS} V_{h+1}^\pi(x') d\mu_h(x').$ 
    
    According to the assumption that $\Vert \theta_{r,h}\Vert \leq \sqrt{d},$ and $\Vert \int_\cS V_{h+1}^\pi(x') d\mu_h(x')\Vert \leq H\sqrt{d},$ the result follows.
\end{proof}

\begin{lemma}\label{le:lambdahk-bound}
    Let $\Lambda_k=\lambda I+\sum_{i=1}^k \phi_i\phi_i^\top,$ where $\phi_i\in\bR^d$ and $\lambda>0.$ Then $$\sum_{i=1}^k\phi_i^\top(\Lambda_k)^{-1}\phi_i\leq d. $$
\end{lemma}
\begin{proof}
    We first have 
    \begin{align*}
       \sum_{i=1}^k\phi_i^\top(\Lambda_k)^{-1}\phi_i 
        =\sum_{i=1}^k tr(\phi_i^\top(\Lambda_k)^{-1}\phi_i) = tr((\Lambda_k)^{-1}\sum_{i=1}^k\phi_i\phi_i^\top).
    \end{align*}
    Using the eigenvalue decomposition $$\sum_{i=1}^k\phi_i\phi_i^\top=U diag(\lambda_1,\dots,\lambda_d)U^\top,$$ we obtain $$\Lambda_k = U diag(\lambda_1+\lambda,\dots,\lambda_d+\lambda)U^\top.$$
    Then we have $$ ((\Lambda_k)^{-1}\sum_{i=1}^k \phi_i\phi_i^\top) ) =\sum_{j=1}^d \lambda_j/(\lambda_j+\lambda)\leq d. $$
\end{proof}
\begin{lemma}\label{le:whk-bound}
    For any $k,h,$ the estimate weight parameter $w^k_h$ satisfies 
    \begin{equation}
        \Vert w_h^k\Vert \leq 2H\sqrt{dk/\lambda}
    \end{equation}
\end{lemma}
\begin{proof}
    For any vector $v\in\bR^d,$ we have
    \begin{align}
    \vert v^\top w_h^k\vert & = \left\vert v^\top(\Lambda_h^k)^{-1}\sum_{\tau=1}^{k-1} \phi_h^\tau\left( r_h(x_h^\tau,a_h^\tau) + \sum_a \pi_{h+1}^k(a\vert x_{h+1}^\tau) Q_{h+1}^k(x_{h+1}^\tau,a) \right) \right\vert \nonumber \\
    & \leq_{(i)} \left\vert  v^\top(\Lambda_h^k)^{-1}\sum_{\tau=1}^{k-1} \phi_h^\tau \right\vert \cdot 2H \nonumber \\
    & \leq  \sum_{\tau=1}^{k-1} \left\vert  v^\top(\Lambda_h^k)^{-1} \phi_h^\tau \right\vert \cdot  2H \nonumber \\
    &\leq_{(ii)}  \sqrt{ \sum_{\tau=1}^{k-1}   v^\top(\Lambda_h^k)^{-1} v }\sqrt{  \sum_{\tau=1}^{k-1}   (\phi_h^\tau) ^\top(\Lambda_h^k)^{-1} \phi_h^\tau } \cdot 2H \nonumber \\
    & \leq_{(iii)}  2H\Vert v\Vert \frac{\sqrt{dk}}{\lambda},
    \end{align}
    where the inequality $(i)$ holds because for any $(x,a)$ we have $Q_{h+1}^k(x,a)\leq H,$  inequality $(ii)$ comes from using Cauchy-Schwarz inequality, and the last inequality $(iii)$ is true dute to Lemma \ref{le:lambdahk-bound}. Note that $\Vert w_h^k\Vert = \max_{v:\Vert v\Vert=1}\vert v^\top w_h^k\vert,$ hence we prove the lemma.  
\end{proof}

\begin{lemma}\label{le:v-ev-bound}
Let $\{x_\tau\}_{\tau=1}^\infty$ be a stochastic process on state space $\cS,$ with corresponding filtration $\{\cF_\tau\}_{\tau=0}^\infty.$ Let $\{\phi_\tau\}_{\tau=0}^\infty$ be an $\bR^d-$valued stochastic process where $\phi_\tau\in\cF_{\tau-1}.$ We assume that $\Vert \phi_\tau\Vert\leq1,$ let $\Lambda_k=\lambda I+\sum_{\tau=1}^k\phi_\tau\phi_\tau^\top.$ For all $K\geq 0,$ and any $V\in\cV,\sup_x\vert V(x)\Vert\leq H,$ we have for any $\delta>0,$ with probability at least $1-\delta,$
\begin{align}
    \left\vert \sum_{\tau=1}^k \phi_\tau \{ V(x_\tau)-\bE[V(x_\tau)\vert \cF_{\tau-1} ]\} \right\Vert_{\Lambda_k^{-1}}^2 
    \leq  4H^2\left[ \frac{d}{2}\log\left(\frac{k+\lambda}{\lambda}\right) + \log\frac{\cN_\epsilon}{\delta}  \right] + \frac{8k^2\epsilon^2}{\lambda},
\end{align}
where $\cN_\epsilon$ is the $\epsilon-$covering number of function class $\cV$ with respect to the distance $$\text{dist}(V,V')=\sup_x\vert V(x)-V'(x)\vert,V\in\cV,V'\in\cV.$$
\end{lemma}
\begin{proof}
We know that for any $V\in\cV,$ there exists a $\hat{V}\in\cV$ in the $\epsilon-$covering such that $$V=\hat{V}+\Delta_V,\quad \sup_x\vert \Delta_V(x)\vert \leq\epsilon.$$ Thus we have the following decomposition:
\begin{align}
  &  \left\Vert \sum_{\tau=1}^k \phi_\tau \{ V(x_\tau)-\bE[V(x_\tau)\vert \cF_{\tau-1} ]\} \right\Vert_{\Lambda_k^{-1}}^2\nonumber\\
    \leq &2\left\Vert \sum_{\tau=1}^k \phi_\tau \{ \hat{V}(x_\tau)-\bE[\hat{V}(x_\tau)\vert \cF_{\tau-1} ]\} \right\Vert_{\Lambda_k^{-1}}^2  
    +2 \left\Vert \sum_{\tau=1}^k \phi_\tau \{ \Delta_V(x_\tau)-\bE[\Delta_{V}(x_\tau)\vert \cF_{\tau-1} ]\} \right\Vert_{\Lambda_k^{-1}}^2.
\end{align}
The first term is bounded by using the concentration of self-normalized processes \cite{AbbPalSze_11}, and the bound of the second term is easy to show as $8k^2\epsilon^2/\lambda.$
\end{proof}
\begin{lemma}\label{le:cover-num}
(Covering Number of Euclidean Ball). For any $\epsilon>0,$ the $\epsilon-$covering number of the Euclidean ball in $\bR^d$ with radius $R>0$ is upper bounded by $(1+2R/\epsilon)^d.$
\end{lemma}
This lemma is a basic result on the covering number of a Euclidean ball. Detailed proof can be found in Lemma $5.2$ in \cite{Ver_10}.

\begin{lemma}
    Let $\cQ$ denote a class of functions mapping from $\cS\times\cA$ to $\bR$ with following parametric form 
    \begin{align}
        \cQ=\{Q\vert Q(\cdot,\cdot) 
        =\min\{w^\top \phi(\cdot,\cdot) + \beta\sqrt{\phi(\cdot,\cdot)^\top \Lambda^{-1}\phi(\cdot,\cdot) },H  \} \}
    \end{align}
 where the parameters $(w,\beta,\Lambda)$ satisfy $\Vert w\Vert\leq L,\beta\in[0,B],$ and the minimum eigenvalue of $\Lambda$ satisfies $\lambda_{\min}(\Lambda)\geq \lambda.$ Let $\cV$  denote a functions mapping from $\cS$ to $\bR,$   with the following parametric form 
    \begin{align}
        \cV:=\{V(\cdot),V(\cdot) = &\max_a (Q(\cdot,a)-Zg(\phi(\cdot,a))^+), Z\in[0,\infty),g\in[-1,L_\delta],Q\in\cQ \}. \label{eq:class-v}
    \end{align}
Further assume that $\Vert \phi(x,a)\Vert\leq 1$ for all $(x,a)$ pairs. Let $\cN_\epsilon$ be the $\epsilon-$covering number of $\cV,$ with respect to the distance, which is denoted as $$\text{dist}(V,V')=\sup_x\vert V(x) - V'(x)\vert.$$ Then 
    \begin{align}
        \log \cN_\epsilon \leq & d\log(2+8L/\epsilon) + d^2\log [4+32d^{1/2}B^2/(\lambda\epsilon^2) ] 
        + d^2 \log  [4+72L/\epsilon^2 ]
    \end{align}
\end{lemma}
\begin{proof}
    For notation simplicity, we represent $A=\beta^2\Lambda^{-1},$ so we have 
    \begin{align}
        \cQ=\{Q\vert Q(\cdot,\cdot)=\min\{w^\top \phi(\cdot,\cdot) + \sqrt{\phi(\cdot,\cdot)^\top A \phi(\cdot,\cdot) },H  \} \}, \label{eq:class-q}
    \end{align}
    for $\Vert w\Vert\leq L,\Vert A\Vert\leq B^2\lambda^{-1}.$ Then for any two functions $V_1,V_2,\in\cV,$ let them take the form in Eq.\eqref{eq:class-v} with parameters $(w_1,A_1,Z_1)$ and $(w_2,A_2,Z_2),$ respectively. Then we have 
    \begin{align}
       & \text{dist}(V_1,V_2) = \sup_x\vert V_1(x) - V_2(x)\vert \nonumber\\
        =& \sup_x \{\vert Q_1(x,a_1^*)-Q_2(x,a_2^*)\vert  \} \nonumber \\
        \leq & \sup_x\{\vert Q_1(x,a_1^*) - Q_2(x,a_1^*)\vert \} + \sup_x\{\vert Q_2(x,a_1^*)- Q_2(x,a_2^*) \vert  \} \nonumber \\
        \leq & \sup_{x,a} \bigg\vert \bigg[  w_1^\top \phi(x,a)+ \sqrt{\phi(x,a)^\top A_1\phi(x,a)} \bigg ]- \bigg[  w_2^\top \phi(x,a)+ \sqrt{\phi(x,a)^\top A_2\phi(x,a)} \bigg ]\bigg\vert  \nonumber \\
        & +  \sup_{x} \bigg\vert \bigg[  w_2^\top \phi(x,a_1^*)+ \sqrt{\phi(x,a_1^*)^\top A_2\phi(x,a_1^*)} \bigg ]- \bigg[  w_2^\top \phi(x,a_2^*)+ \sqrt{\phi(x,a_2^*)^\top A_2\phi(x,a_2^*)} \bigg ]\bigg\vert        \nonumber\\
        \leq & \sup_{\phi:\Vert\phi\Vert\leq 1}\left\vert  \left[w_1^\top\phi+\sqrt{\phi^\top A_1\phi} \right] - \left[ w_2^\top \phi + \sqrt{\phi^\top A_2\phi} \right]  \right\vert     \nonumber \\
        &+ \sup_{x} \vert w_2^\top (\phi(x,a_1^*)-\phi(x,a_2^*))\vert  + \sup_{x} \left| \sqrt{\phi(x,a_1^*)^\top A_2\phi(x,a_1^*)} - \sqrt{\phi(x,a_2^*)^\top A_2\phi(x,a_2^*)}   \right|   \nonumber \\ 
        \leq & \sup_{\phi:\Vert\phi\Vert\leq 1} \big| (w_1-w_2)^\top \phi\big| +  \sup_{\phi:\Vert\phi\Vert\leq 1} \sqrt{\big| \phi^\top (A_1-A_2)\phi^\top \big|} +  \sup_{x} \Vert w_2\Vert \Vert (\phi(x,a_1^*)-\phi(x,a_2^*))\Vert \nonumber \\
        &+ \sup_{x}\sqrt{\left| (\phi(x,a_1^*)^\top -\phi(x,a_2^*)^\top) A_2 \phi(x,a_1^*) - (\phi(x,a_2^*)^\top-\phi(x,a_1^*)^\top)A_2\phi(x,a_2^*)    \right|   }       \nonumber \\
        \leq & \Vert w_1-w_2\Vert + \sqrt{\Vert A_1-A_2\Vert} + \sup_{x} L\cdot \Vert (\phi(x,a_1^*)-\phi(x,a_2^*))\Vert  \nonumber \\
        &+ \sup_x\sqrt{\Vert \phi(x,a_1^*)^\top -\phi(x,a_2^*)^\top\Vert\Vert A_2\Vert    } +  \sup_x\sqrt{\Vert \phi(x,a_2^*)^\top -\phi(x,a_1^*)^\top\Vert \Vert A_2\Vert    }  \nonumber \\
        \leq & \Vert w_2-w_2\Vert + \sqrt{\Vert A_1-A_2\Vert_F} + \sup_{x} L\cdot \Vert (\phi(x,a_1^*)-\phi(x,a_2^*))\Vert\nonumber \\
         &+ 2\sup_x\sqrt{\Vert \phi(x,a_1^*)^\top -\phi(x,a_2^*)^\top\Vert \Vert A_2\Vert_F    } ,
    \end{align}
    where $a_i^*=\arg\max_a(Q_i(x,a)-Z_i g(\phi(x,a) )_+ ),$ the third inequality follows from the fact that $\vert \sqrt{x} -\sqrt{y}\vert \leq \sqrt{\vert x-y\vert}, x,y\geq 0.$ We use $\Vert \cdot\Vert$ and $\Vert \cdot\Vert_F$ to denote the matrix operator norm and Frobenius norm respectively. 

    Let $\cC_w$ be an $\epsilon/2-$cover of $\{w\in\bR^d\vert \Vert w\Vert\leq L \}$ with respect to the $2-$norm, we use $\cC_A$ to denote the $\epsilon^2/4-$cover of $\{A\in\bR^{d\times d}\vert \Vert A\Vert_F\leq d^{1/2}B^2\lambda^{-1} \}$ with respect to the Frobenius norm, and let $\cC_\phi$ be an $\epsilon^2/9(\min\{L,d^{1/2}B^2\lambda^{-1} \})-$cover of $\{\phi \in \bR^d \vert \Vert \phi\Vert \leq 1 \}$. By using Lemma \ref{le:cover-num} in \cite{JinYanZha_20}, we know that: 
    \begin{align}
        \vert \cC_w\vert \leq (1+4L/\epsilon)^d,\quad \vert \cC_A\vert \leq [1+8d^{1/2}B^2/(\lambda\epsilon^2) ]^{d^2}, \quad \vert \cC_\phi\vert \leq [1+18L/\epsilon^2 ]^{d^2}  .
    \end{align}
    Then we have for any $V_1\in\cV,$ there exists  $V_2$ parameterized by $(w_2\in \cC_w, A_2\in\cC_A, Z_2)$ such that $\text{dist}(V_1,V_2)\leq 2\epsilon.$ Hence it holds that $\cN_{2\epsilon}\leq \vert \cC_w\vert \cdot \vert \cC_A\vert \cdot \vert \cC_\phi\vert, $ which gives:
    \begin{align}
        \log \cN_{2\epsilon} \leq &\log\vert \cC_w\vert +\log\vert \cC_A\vert  + \log\vert \cC_\phi\vert \nonumber \\
        \leq & d\log(1+4L/\epsilon) + d^2\log [1+8d^{1/2}B^2/(\lambda\epsilon^2) ] + d^2 \log  [1+18L/\epsilon^2 ]
    \end{align}
    Rescaling $\epsilon$ to $\epsilon/2,$ we prove the lemma.
\end{proof}

\section{Proof of technical Lemmas}
\subsection{Proof of Lemma \ref{le:err-ga}}
\begin{proof}
    Using Lemma $3$ in \cite{ChoGop_17}, we have 
    $$\gamma_h^k \geq \frac{1}{2}\sum_{\tau=1}^k \log(1+\lambda^{-1}(\sigma_h^\tau(y_h^\tau)) )^2,y_h^\tau= (x_h^\tau,a_h^\tau).$$
    By Cauchy-Schwartz inequality, we have $\sum_{\tau=1}^K \sigma_h^\tau(y_h^\tau)\leq \sqrt{K\sum_{\tau=1}^k(\sigma_h^{\tau}(y_h^\tau))^2}.$ Since $0\leq (\sigma_h^\tau(y_h^\tau))^2\leq 1,$ for all $y,$ we also have $\lambda^{-1}(\sigma_h^\tau(y_h^\tau))^2\leq 2 \ln(1+\lambda^{-1}(\sigma_{h}^\tau(y_h^\tau))^2),$ and using the fact that $0\leq\alpha\leq 1, ln(1+\alpha)\geq \alpha/2,$ we can obtain
    \begin{align}
        \sum_{\tau=1}^k \sigma_h^\tau(y_h^\tau) \leq \sqrt{2K\sum_{\tau=1}^K\lambda \ln(1+\lambda^{-1}(\sigma_{h}^\tau(y_h^\tau))^2) } \leq \sqrt{4K\lambda\sum_{\tau=1}^K \frac{1}{2}\ln (1+ (\sigma_{h}^\tau(y_h^\tau))^2))  }\leq \sqrt{4K\gamma_h^K}
    \end{align}
    Since $\beta_h^k(p/H)$ is increasing with episode $k,$ then we have 
    $$\sum_{\tau=1}^K \beta_h^\tau(p/H)\sigma_h^\tau(x_h^\tau)\leq \beta_h^K(p/H)\sqrt{4(K+2)\gamma_h^K } .$$ Substituting the definition of $\beta_h^K(p/H)$ we prove the result.
\end{proof}
\subsection{Proof of Lemma \ref{le:err-linear}}
Recall that for the linear case we have 
$$g_h(x,a) - \hat{g}_h^k(x,a)\leq e_h^k(p,x,a)= 2\tilde\beta_h^k(p/H)\Vert\phi(x,a)\Vert_{(\Lambda_h^k)^{-1}}.$$
Then 
\begin{align*}
   & \sum_{k=1}^K\sum_{h=1}^H e_h^k(p,x,a) = \sum_{k=1}^K\sum_{h=1}^H 2\tilde\beta_h^k(p/H)\Vert\phi(x,a)\Vert_{(\Lambda_h^k)^{-1}}\\
   \leq &\sum_{k=1}^K\sum_{h=1}^H2\tilde\beta_h^k(p/H) \min\{\Vert\phi(x,a) \Vert_{(\Lambda_h^k)^{-1}},1 \} \\
   \leq & \sum_{h=1}^H2 \tilde\beta_h^K(p/H) \sqrt{ K\sum_{k=1}^K \min\{\Vert\phi(x,a) \Vert_{(\Lambda_h^k)^{-1}}^2,1 \} } \\
    \leq & \sum_{h=1}^H \tilde\beta_h^K(p/H)  \sqrt{ 8dK\log \bigg( \frac{d+K}{d}\bigg)  },
\end{align*}
where the first inequality is true because $\Vert \phi(x,a)\Vert\leq 1$ by assumption, and the last inequality is using the following Elliptical Potential Lemma (Theorem $11.7$ in \cite{CesLug_06}, Lemma $11$ in \cite{AbbPalSze_11} and Theorem $19.4$ in \cite{LatSze_20}).
\begin{lemma}
    Let $\Lambda_0=I,$ and $\phi_0,\dots,\phi_t\in\bR^d$ be a sequence of vector with $\Vert \phi_t\Vert \leq 1$ for any $t,$ and $\Lambda_t = I + \sum_{\tau=1}^t \phi_\tau\phi_\tau^\top,$ then,
    $$\sum_{\tau=1}^t\min\{1,\Vert \phi_\tau\Vert_{\Lambda_\tau^{(-1)}}^2 \}\leq 2\log\bigg(\frac{\text{det}\Lambda_t}{\text{det} \Lambda_0}\bigg)\leq 2d\log\bigg(\frac{d+t}{d}\bigg).$$
\end{lemma}
Substituting the definition of $\tilde\beta_h^K(p/H),$ we prove the result.
\subsection{Proof of Lemma \ref{le:over-est}}
\begin{proof}
    First, for the step $H.$ It holds obviously, since $Q_{H+1}^k(x,a)=Q_{H+1}^*(x,a)=0,\forall x,a.$ Now suppose that it is true till step $h+1$ and consider step $h.$ Then we have for all $(x,a),$
    \begin{align}
  Q_{h+1}^k(x,a) -Z^k_{h+1} \hat{g}^k_{h+1}(x,a)_+ \geq Q_{h+1}^*(x,a) - Z^k_{h+1} \hat{g}^k_{h+1}(x,a)_+
    \end{align}
    Then \begin{align}
         & Q_{h+1}^k(x,a^*) -Z^k_{h+1} \hat{g}^k_{h+1}(x,a^*)_+ \nonumber \\
         \geq & Q_{h+1}^*(x,a^*) - Z^k_{h+1} \hat{g}^k_{h+1}(x,a^*)_+  \nonumber \\
         \geq &  Q_{h+1}^*(x,a^*) - Z^k_{h+1} {g}_{h+1}(x,a^*)_+ ,
    \end{align}
        where $a^*$ is the action selected by the optimal policy, and the last inequality is true due to Lemma \ref{le:under_est}. Then we can obtain that 
$$ \max_{a} \{Q_{h+1}^k(x,a) -Z^k_{h+1} \hat{g}^k_{h+1}(x,a)_+\} \geq Q_{h+1}^k(x,a^*) -Z^k_{h+1} \hat{g}^k_{h+1}(x,a^*)_+ \geq Q_{h+1}^*(x,a^*) = V_{h+1}^*(x).$$ Therefore we have 
\begin{equation}
     V_{h+1}^k(x) \geq  V_{h+1}^*(x).
\end{equation}
Thus we have $$\bP_h(V_{h+1}^* - V_{h+1}^k )(x,a) \leq  0.$$ 
According to Lemma \ref{le:error-bound}, we know that: 
\begin{align}
    \vert\langle \phi(x,a),w_h^k\rangle - Q_h^*(x,a) - \bP_h(V_{h+1}^k - V_{h+1}^*)(x,a)\vert \leq \beta\sqrt{\phi(x,a)^\top (\Lambda_h^k)^{-1}\phi(x,a)}. 
\end{align}
Then we can obtain
\begin{align}
    Q_h^*(x,a) \leq \min\{ \langle \phi(x,a) ,w_h^k\rangle +\beta \sqrt{\phi(x,a)^\top (\Lambda_h^k)^{-1}\phi (x,a)}, H \}=Q_h^k(x,a).
\end{align}
    We finish proving the lemma.
\end{proof}

\subsection{Proof of Lemma \ref{le:vk-vpik}}
First, we present the concentration lemma, which controls the fluctuations in the least square value iteration.
\begin{lemma}\label{le:concentration}
   Given the constant $\beta=c_\beta\cdot dH\sqrt{\iota},\iota=\log(2dKH/p),$ defined in out algorithm \ref{alg:hard}, there exists a constant $C$ such that for any fixed $p\in(0,1),$ if we let $\cE$ be the event that:
   \begin{align}
    \forall (k,h)\in[K]\times[H]: \quad \bigg\| \sum_{\tau=1}^{k-1}\phi_h^\tau[V_{h+1}^k(x_{h+1}^\tau)-\bP_h V_{h+1}^k(x_h^\tau,a_h^\tau)  ]\bigg\|_{(\Lambda_h^k)^{-1}} \leq C\cdot dH\sqrt{\chi},
   \end{align}
   where $\chi=log[3(c_\beta+1)dHK/p ],$ then $\bP(\cE)\geq 1-p/2.$
\end{lemma}
\begin{proof}
    We know that for all $k\in[K],h\in[H],$ we have $\Vert w_h^k\Vert \leq 2H\sqrt{dk/\lambda}$ from Lemma \ref{le:whk-bound}. Using Lemma \ref{le:v-ev-bound} and Lemma \ref{le:cover-num} and, we have for any $\epsilon>0,$
    \begin{align}
     &   \bigg\| \sum_{\tau=1}^{k-1}\phi_h^\tau[V_{h+1}^k(x_{h+1}^\tau)-\bP_h V_{h+1}^k(x_h^\tau,a_h^\tau)  ]\bigg\|_{(\Lambda_h^k)^{-1}}^2 \nonumber \\
        \leq &  4H^2\bigg[ \frac{d}{2}\log\left(\frac{k+\lambda}{\lambda}\right)  + d\log\left(2+\frac{16H\sqrt{dk}}{\epsilon\sqrt{\lambda}} \right) + d^2\log \left(4+ \frac{32d^{1/2}\beta^2}{\epsilon^2\lambda}\right) \nonumber \\
        &+ d^2\log \left(\frac{4+144H\sqrt{dk}}{\epsilon^2\sqrt{\lambda}} \right) + \log\left(\frac{2}{p}\right)  \bigg] + \frac{8k^2\epsilon^2}{\lambda},
    \end{align}
    Since in our algorithm, we choose $\lambda=1,\beta=c_\beta\cdot dH\sqrt{\iota},$ then by plugging $\epsilon=dH/k,$ we conclude that there exists a absolute constant $C>0$ which is independent of $c_\beta,$ such that 
    $$  \bigg\| \sum_{\tau=1}^{k-1}\phi_h^\tau[V_{h+1}^k(x_{h+1}^\tau)-\bP_h V_{h+1}^k(x_h^\tau,a_h^\tau)  ]\bigg\|_{(\Lambda_h^k)^{-1}}^2 \leq C\cdot d^2H^2\log[3(c_\beta+1)dHK/p].$$
\end{proof}
Next, we provide a recursive lemma, which bounds the difference between the estimate $Q-$value function $Q_h^k$ and the true $Q-$value function $Q_h^\pi$ of any given policy $\pi$ with high probability.
\begin{lemma}\label{le:error-bound}
   Given the constant $\beta=c_\beta\cdot dH\sqrt{\iota},\iota=\log(2dKH/p),$ defined in out algorithm \ref{alg:hard}. For any given policy $\pi,$ on the event $\cE$ defined in Lemma \ref{le:concentration}, we have for all $(x,a,h,k)\in\cS\times \cA \times [H]\times [K]$ that: 
   \begin{align}
       \langle \phi(x,a), w_h^k\rangle -Q_h^\pi(x,a) = \bP_h(V_{h+1}^k - V_{h+1}^\pi) (x,a) + \Delta_h^k(x,a),
   \end{align}
   for some $\Delta_h^k(x,a),$ that satisfies $\vert \Delta_h^k(x,a)\vert \leq \beta\sqrt{\phi(x,a)^\top (\Lambda_h^k)^{-1}\phi(x,a) }.$
\end{lemma}
\begin{proof}
    Let $r_h^\tau := r_h(x_h^\tau,a_h^\tau).$ We first know that for any $(x,a,h)\in\cS\times\cA\times[H],$ $$Q_h^\pi(x,a) := \langle \phi(x,a),w_h^\pi\rangle = r_h(x,a)+\bP_h V_{h+1}^\pi(x,a).$$ Then we have:
    \begin{align}
        w_h^k - w_h^\pi = & (\Lambda_h^k)^{-1}\sum_{\tau=1}^{k-1}\phi_h^\tau[r_h^\tau + V_{h+1}^k(x_{h+1}^\tau) ] - w_h^\pi\nonumber\\
        = & \underbrace{-\lambda(\Lambda_h^k)^{-1}(w_h^\pi)}_{\bf{q1}} + \underbrace{(\Lambda_h^k)^{-1}\sum_{\tau=1}^{k-1} \phi_h^\tau [ V_{h+1}^k(x_{h+1}^\tau )-\bP_h V_{h+1}^k(x_h^\tau,a_h^\tau)]}_{\bf{q2}} \nonumber \\
        & + \underbrace{(\Lambda_h^k)^{-1}\sum_{\tau=1}^{k-1} \phi_h^\tau [ \bP V_{h+1}^k(x_{h}^\tau,a_h^\tau )-\bP_h V_{h+1}^\pi (x_h^\tau,a_h^\tau)]}_{\bf{q3}}. \label{eq:qterms}
    \end{align}
    Now we bound each time on the right-hand side of the expression in Eq. \eqref{eq:qterms} individually. For the first term,
    \begin{align}
        \vert \langle \phi(x,a), {\bf{q1}}\rangle\vert = \vert \lambda\langle \phi(x,a), (\Lambda_h^k)^{-1} w_h^\pi \rangle\vert \leq \sqrt{\lambda }\Vert w_h^\pi\Vert \sqrt{\phi(x,a)^\top (\Lambda_h^k)^{-1}\phi(x,a) }. \label{eq:q1-bound}
    \end{align}
    For the second term, given the event $\cE$ defined in Lemma \ref{le:concentration}, we have:
    \begin{align}
        \vert \langle \phi(x,a),{\bf{q2}}\rangle \vert \leq c_0\cdot dH\sqrt{\chi}\sqrt{\phi(x,a)^\top (\Lambda_h^k)^{-1}\phi(x,a)} \label{eq:q2-bound}
    \end{align} for an absolute constant $c_0$ independent of $c_\beta,$ and $\chi=\log[3(c_\beta+1)dHK/p ].$ For the third term,
    \begin{align}
        &  \vert \langle \phi(x,a),{\bf{q3}}\rangle \vert = \bigg\langle \phi(x,a), (\Lambda_h^k)^{-1} \sum_{\tau=1}^{k-1}\phi_h^\tau \bP_h(V_{h+1}^k - V_{h+1}^\pi)(x_h^\tau,a_h^\tau)\bigg\rangle \nonumber \\
         = &  \bigg\langle \phi(x,a), (\Lambda_h^k)^{-1} \sum_{\tau=1}^{k-1}\phi_h^\tau (\phi_h^\tau)^\top \int (V_{h+1}^k - V_{h+1}^\pi)(x') d \mu_h(x')   \bigg\rangle \nonumber \\
         = &  \bigg\langle  \phi(x,a), \int (V_{h+1}^k - V_{h+1}^\pi)(x') d \mu_h(x')   \bigg\rangle -  \bigg\langle  \phi(x,a), \lambda(\Lambda_h^k)^{-1}  \int (V_{h+1}^k - V_{h+1}^\pi)(x') d \mu_h(x')   \bigg\rangle.\label{eq:q3terms}
    \end{align}
    The first term in Eq. \eqref{eq:q3terms} is equal to $$\bP_h(V_{h+1}^k -V_{h+1}^\pi )(x,a).$$ 
    The second term in Eq. \eqref{eq:q3terms} can be bounded as:
    \begin{align}
        \bigg\vert  \bigg\langle  \phi(x,a), \lambda(\Lambda_h^k)^{-1}  \int (V_{h+1}^k - V_{h+1}^\pi)(x') d \mu_h(x')   \bigg\rangle \bigg\vert \leq 2H\sqrt{d\lambda} \sqrt{\phi(x,a)^\top (\Lambda_h^k)^{-1}\phi(x,a)}. \label{eq:q3-bound}
    \end{align}
    Since $\langle \phi(x,a),w_h^k\rangle - Q_h^\pi(x,a) = \langle \phi(x,a), w_h^k-w_h^\pi \rangle + \langle \phi(x,a), {\bf{q_1+q_2+q_3}} \rangle,$ then combing the results from Equations \eqref{eq:q1-bound},\eqref{eq:q2-bound},\eqref{eq:q3-bound}, Lemma \ref{le:whpi-bound} and our choice of parameter $\lambda=1,$ we can obtain 
\begin{align}
    \vert \langle \phi(x,a),w_h^k\rangle - Q_h^\pi(x,a) -\bP_h(V_{h+1}^k -V_{h+1}^\pi )(x,a)\vert \leq c'\cdot dH\sqrt{\chi}\sqrt{\phi(x,a)^\top (\Lambda_h^k)^{-1}\phi(x,a)},
\end{align}
    for an absolute constant $c'$ independent of $c_\beta.$ Finally, to prove this lemma, we only need to show that  
    \begin{align}
        c'\sqrt{\chi} = & c' \sqrt{\log[ 3 (c_\beta+1) dHK/p ] } \nonumber \\
        = & c'\sqrt{\iota + \log(\frac{3}{2}(c_\beta+1))} \nonumber \\
        \leq & c_\beta\sqrt{\iota} \label{eq:c'cbeta}
    \end{align}
    where $\iota= \log(2 dHK/p).$ We know that $\iota \in [\log 2,\infty  ] $ and $c'$ is an absolute constant independent of $c_\beta.$ Therefore, by choosing $c_\beta$ to make it satisfy $c'\sqrt{\log 2 + \log (3/2(c_\beta+1))  }\leq c_\beta\sqrt{\log 2},$ we can observe that Eq.\ref{eq:c'cbeta} holds for all $\iota\in[\log 2,\infty).$ 
\end{proof}

\begin{lemma}
Under the event defined in Lemma \ref{le:concentration}, for any fixed $p\in(0,1),$ if we set $\lambda=1,\beta=c\cdot dH\sqrt{\iota}$ in Algorithm \ref{alg:hard} with $\iota=\log(2dHK/p),$ then with probability at least $1-p/2,$ we have :
    \begin{align}
     \sum_{k=1}^K V_h^k(x_h^k)-V^{\pi^k}_h(x_h^k) \leq c'\sqrt{d^3H^4K\iota^2} = {\cO}(\sqrt{d^3H^4K\iota^2}).
    \end{align}
\end{lemma}
\begin{proof}
    Let $\delta_h^k = V_h^k(x_h^k) - V_h^{\pi^k}(x_h^k),$ and $\zeta_{h+1}^k = \bE[\delta_{h+1}^k\vert x_h^k,a_h^k ]-\delta_{h+1}^k.$ Then on the event defined in Lemma \ref{le:concentration}, then for and $k\in[K],h\in[H],$ according to Lemma \ref{le:error-bound} we know that for $(x_h^k,a_h^k),$
    \begin{align}
     \delta_h^k= &  V_h^k(x_h^k ) - V_h^{\pi^k}(x_h^k)=  Q_h^k(x_h^k,a_h^k) - Q_h^{\pi^k}(x_h^k,a_h^k) \nonumber \\
     \leq &\bP_h(V_{h+1}^k - V_{h+1}^{\pi^k}) (x_h^k,a_h^k) +  \beta\sqrt{(\phi_h^k)^\top (\Lambda_h^k)^{-1}\phi_h^k} \nonumber \\
        \leq & \delta_{h+1}^k + \zeta_{h+1}^k + \beta\sqrt{(\phi_h^k)^\top (\Lambda_h^k)^{-1}\phi_h^k}.
    \end{align}
Then taking the summation over $K$ episodes on both sides, we have
\begin{align}
  &  \sum_{k=1}^K[Q_1^k(x_1^k,a_1^k)  - Q_h^{\pi^k}(x_1^k,a_1^k) ] =   \sum_{k=1}^K[V_1^k(x_1^k)  - V_h^{\pi^k}(x_1^k) ] \nonumber\\
   = & \sum_{k=1}^K \delta_1^k  \leq  \sum_{k=1}^K\sum_{h=1}^H \zeta_h^k + \beta\sum_{k=1}\sum_{h=1}^h \sqrt{(\phi_h^k)^\top (\Lambda_h^k)^{-1}\phi_h^k }
\end{align}
For the first term $\{\zeta_h^k\}$ is a martingale difference sequence satisfying $\vert \zeta_h^k\vert \leq 2H$ for all $(k,h),$ then using Azuma-Hoeffding inequality, for any $t>0,$ with probability at least $1-p/2,$ we have 
\begin{align}
    \sum_{k=1}^K\sum_{h=1}^H\zeta_h^k\leq \sqrt{2KH\cdot H^2\log(2/p)} \leq H\sqrt{HK \iota},  \label{eq:esterror-term1}
    \end{align}
    where $\iota=\log(2dHK/p).$ For the second term, using Lemma $D.2$ in \cite{JinYanZha_20} we have for any $h\in[H],$ 
    \begin{align}
        \sum_{k=1}^K(\phi_h^k)^\top (\Lambda_h^k)^{-1}\phi_h^k \leq 2 \log\left[\frac{det(\Lambda_h^{k+1})}{det(\Lambda_h^1)}. \right]
    \end{align}
    We know that $\Vert \Lambda_h^{k+1}\Vert = \Vert \sum_{\tau=1}^k \phi_h^k(\phi_h^k)^\top + \lambda I \Vert \leq \lambda + k,$ which implies that 
    \begin{align}
        \sum_{k=1}^K (\phi_h^k)^\top (\Lambda_h^k)^{-1}\phi_h^k \leq 2d\frac{\lambda+1}{\lambda} \leq 2d\iota.
    \end{align}
    Using Cauchy-Schwartz inequality, we can further obtain
    \begin{align}
        \sum_{k=1}^K\sum_{h=1}^H \sqrt{  (\phi_h^k)^\top (\Lambda_h^k)^{-1}\phi_h^k  } \leq \sum_{h=1}^H \sqrt{K} \left[ \sum_{k=1}^K  (\phi_h^k)^\top (\Lambda_h^k)^{-1}\phi_h^k \right]^{1/2} \leq H\sqrt{2dK\iota}. \label{eq:esterror-term2}
    \end{align}
    Combing Eq. \eqref{eq:esterror-term1} and Eq. \eqref{eq:esterror-term2} and with our choice of $\beta=cdH\sqrt{\iota},$ we conclude that with probability at least $1-p:$
    \begin{align}
        \sum_{k=1}^K  Q_h^k(x_h^k,a_h^k)-Q^{\pi^k}_h(x_h^k,a_h^k) \leq H\sqrt{HK\iota} + \beta H\sqrt{2dK\iota} \leq c'\sqrt{d^3H^4K\iota^2}
    \end{align}
\end{proof}


\subsection{Proof of Theorem \ref{the:lower}}
Motivated by \cite{ZhoGuSze_21,HuCheHua_22}, in which they show that a hard-to-learn linear MDP has a lower bound of $\Omega(HD\sqrt{HK}).$ We will illustrate a similar CMDP and then present the specific linear parametrization. 

\noindent {\bf Hard CMDP Instance:} The CMDP instance is denoted as $\cM= \{\cS,\cA,H,\{\bP_h\}_{h=1}^H,\{r_h\}_{h=1}^H,\{g_h\}_{h=1}^H \}.$ The state space $\cS$ consist of states $x_1,\dots,x_{H+2}$ such that $\vert \cS\vert = H+2.$ There are total $2^{d-1}$ actions and $\cA=\{-1,1\}^{d-1}$ such that each action $a\in\cA$ is denoted in vector form.
\begin{itemize}
\item Reward: For any step $h\in[H+2],$ only transitions originating at $x_{H+2}$ incurs a reward.
\item Cost: All actions at state $x_{H+1}, x_{H+2}$ are safe. For any other states, only the transition leads to the highest probability to $x_{H+2}$ is safe.
    \item Transition: $x_{H+1}$ and $x_{H+2}$ are absorbing states regardless of what action is taken. For state $s_i,i\in[H],$ the transition probability is given as 
       \begin{align}
        \bP_h(x'\vert x_i,a) =   \begin{cases}
    \delta + \langle u_h,a\rangle,   \quad  & x'=x_{H+2} \\
    1 - \langle  \delta + \langle u_h,a\rangle, \quad & x'= x_{i+1}, \\
    0 \quad & \text{otherwise}
\end{cases}
    \end{align}
    where $\delta= 1/H$ and $u_h\in\{-\Delta,\Delta\}^{d-1},$ with $\Delta = \sqrt{\delta/K}/(4\sqrt{2})$ 
\end{itemize}
{Linear Parametrization:} Then we will present the linear parametrization of this CMDP. Given the definition, for any $h\in[H],$ the transition probability matrix $\bP_h$ is defined as $$\bP_h(x'\vert x,a ) = \langle \phi(x,a),\mu_h(x')\rangle.$$ The reward function is defined as: $r_h(x,a)=\langle \phi(x,a),\theta_h\rangle,$ where $\phi(x,a)$ is a known feature mapping. $\mu_h = (\mu_h(x_1),\dots,\mu_h(x_{H+2}))\in\bR^{(d+1)\times(H+2)} $ and $\theta_h\in\bR^{d+1}$ are unknown parameters in linear CMDPs. Here $\phi(x,a),\mu_h,\theta_h$ are specified as:
\begin{align}
\phi(x,a) = & \begin{cases}
    (\alpha,\beta a^\top,0 )^\top, \quad  & x=x_h,h\in[H+1] \\
    (0,0^\top,1)^\top, \quad & x= x_{H+2} 
\end{cases}\\
\mu_h(x') = & \begin{cases}
    ((1-\delta)/\alpha, -u_h^\top / \beta, 0 )^\top, \quad & x'=x_{H+1} \\
    (\delta/\alpha,u_h^\top / \beta,1)^\top, \quad & x'=x_{H+2} \\
    0, \text{otherwise}
\end{cases}\\
\theta_h & = (0^\top,1)^\top ,\\
g_h(\phi(x,a)) = & \begin{cases}
    0 \quad & x=x_{H+2},x=_{H+1} \\
    0 \quad & x=x_h,h\in[H+1],a = \argmax_{a'} \langle u_h,a'\rangle \\
    1 \quad & \text{otherwise}\\
\end{cases}
\end{align}
where $\alpha=\sqrt{1/(1+\Delta(d+1))},$ and $\beta=\sqrt{\Delta/(1+\Delta(d-1))}.$ 
Given the definitions, we can easily check the Assumption \ref{as:linear} is satisfied. In particular, we have 
\begin{itemize}
    \item For $x=x_h,h\in[H+1],$ we have $\Vert \phi(x,a)\Vert_2 = \sqrt{\alpha^2+(d-1)\beta^2}=1.$ For $x_{H+2},$ we have $\Vert\phi(x_{H+2},a) \Vert_2=1.$ Thus $\Vert \phi(x,a)\Vert_2\leq 1,\forall x\in\cS,a\in \cA.$
    \item For any $v=(v_1,\dots,v_{H+2})\in \bR^{H+2}$ such that $\vert v\Vert_\infty \leq 1,$ we have
    $$\vert \mu_h  v\Vert_2^2=\left[ \frac{v_1(1-\delta)}{\alpha} + \frac{v_1\delta}{\alpha}  \right]^2 + v_{H+2}^2 \leq \frac{1}{\alpha^2} + 1 = [1+\Delta(d-1)]^2+1=\left[1+ \frac{\sqrt{\delta/K}}{4\sqrt{2}}(d-1) \right]^2+1\leq d+1,  $$
    where the last inequality holds due the assumption that $K\geq (d-1)/(32H(\sqrt{d}-1)).$ Therefore we also have $\Vert \mu_h v\Vert_2\leq \sqrt{d+1},h\in[H].$
    \item For the $\theta_h,$ it is obvious that $\Vert \theta_h\Vert_2\leq 1\leq \sqrt{d+1},h\in[H].$
\end{itemize}

Now, clearly, since the only rewarding transitions are those from $x_{H+2},$ and the definition of the cost functions, the optimal strategy in stage $h$ when in state $x_h$ is to take action $\argmax_{a\in\cA}\langle u_h,a\rangle.$

We first restate a Lemma which shows that the regret in this CMDP can be lower bounded by the regret of $H/2$ bandit instances:
\begin{lemma}
    Suppose $H\geq 3$ and $3(d-1)\Delta\leq \delta.$ Fix $u\in(\{-\Delta,\Delta\}^{d-1})^H.$ Fix a policy $\pi$ and define $\bar{a}_h^\pi = \bE_{u}[ a_h\vert s_h=x_h,s_1=x_1 ] $ which indicates the action taken by the policy when it visits state $x_h$ in step $h$ with the initial state $x_1.$ Let $V^*,V^\pi$ be the optimal value function and value function under policy $\pi$ respectively, then we have
    \begin{align}
        V_1^*(x_1) - V_1^\pi(x_1) \geq & \frac{H}{10} \sum_{h=1}^{H/2} \left( \max_{a\in\cA}\langle u_h,a\rangle - \langle u_h,\bar{a}_h^\pi\rangle
  \right) \\
    \sum_{h=1}^{H}g_h(s_h,a_h)_+  \geq & \sum_{h=1}^{H/2}(1-\frac{4}{3H})\bI_{\{\bar{a}_h^\pi \neq  \max_{a\in\cA}\langle u_h,a\rangle  \}}
    \end{align}
\end{lemma}
\begin{proof}
    The proof of the regret can be found in \cite{ZhoGuSze_21}(Lemma $C.7$). We only provide proof of the violation for completeness. Fix $u,$ and then we can drop the subscript from $\bP$ and $\bE.$ Recall that $\cA=\{+1,-1\}^{d-1}$ and $u_h\in \{-\Delta,+\Delta\}^{d-1}.$ Let $N_h$ denote the event of visiting state $x_h$ in step $h$ then moves to the absorbing state $x_{H+2}:$
    \begin{align}
        N_h = \{s_{h+1}=x_{H+2},s_h=x_h\}.
    \end{align}
    Therefore we can write down the violation for one episode as 
    \begin{align}
        \text{Violation} = \sum_{h=1}^H g_h(s_h,a_h)_+ = \bI_{\{a_1\neq a_1^*\}} + \sum_{h=2}^{H} (1-  \bP(N_{h-1} \vert s_1=x_1)) \bI_{\{a_h\neq a_h^* \}}  ,
    \end{align}
    where we use $a_h^*$ to denote the action chosen by the optimal policy. By the law of total probability, we have that 
    \begin{align}
        &\bP(s_{h+1}=x_{H=2} \vert s_h=x_h,s_1=x_1) \nonumber\\
        =& \sum_{a\in\cA} \bP(s_{h+1}=x_{H+2}\vert s_h=x_h,a_h=a)\bP(a_h=a\vert s_h=x_h,s_1=x_1) \nonumber\\
        = & \sum_{a\in\cA} (\delta + \langle u_h,a\rangle)\bP(a_h=a\vert s_h=x_h,s_1=x_1) \nonumber \\
        = & \delta + \langle u_h,\bar{a}_h^\pi\rangle,
    \end{align}
    where the last equality holds because of the definition that $\bar{a}_h^\pi=\sum_{a\in\cA}\bP(a_h=a\vert s_h=x_h,s_1=x_1)a.$ Also we have that $P(s_{h+1}=x_{h+1} \vert s_h=x_h,s_1=x_1) = 1- (\delta+\langle u_h,\bar{a}_h^\pi\rangle),$ therefore we have:
    \begin{align}
        \bP(N_h) = & (\delta+\langle u_h,\bar{a}_h^\pi\rangle) \prod_{j=1}^{h-1}(1-\delta-\langle u_j,\bar{a}_j^\pi\rangle) 
    \end{align}
    Define $a_h=\langle u_h,\bar{a}_h^\pi\rangle,$ then we have 
    $$\sum_{h=1}^{H}g_h(s_h,a_h)_+=   \bI_{\{a_1\neq a_1^*\}}  + \sum_{h=2}^H  (1- (a_{h-1}+\delta)\prod_{j=1}^{h-1}(1-a_j-\delta))\bI_{\{a_h\neq a_h^* \}}.$$
Since we have that $$ (a_{h-1}+\delta)\prod_{j=1}^{h-1}(1-a_j-\delta) \leq (d-1)\Delta +\delta \leq \frac{4}{3}\delta,$$ where the last inequality holds because $a_{h-1}\leq (d-1)\Delta.$ Therefore by substituting $\delta=1/H,$ we have 
\begin{align}
    \sum_{h=1}^{H}g_h(s_h,a_h)_+ \geq \bI_{\{a_1\neq a_1^*\}} + \sum_{h=2}^H(1-\frac{4}{3H})\bI_{\{a_h\neq a_h^* \}} \nonumber \\
    \geq \sum_{h=1}^{H/2}(1-\frac{4}{3H})\bI_{\{\bar{a}_h^\pi \neq  \max_{a\in\cA}\langle u_h,a\rangle  \}} .
\end{align}
 \end{proof}
\noindent The next Lemma (Lemma $C.8$ in \cite{ZhoGuSze_21}) provides a lower bound on linear bandits with the action set $\cA=\{-1,1\}^{d-1}.$
\begin{lemma}
    For a positive real $0<\delta\leq 1/3,$ and assume that $K\geq d^2.$ Let $\Delta = \sqrt{\delta/K}(4\sqrt{2})$ and consider the linear bandit problem $\cL_u,$ which is parameterized with a parameter vector $u\in\{-\Delta,\Delta \}^d$ and an action set $\cA\in\{-1,1\}^d$ so that the reward distribution for taking action $a\in\cA$ is a Bernoulli distribution $B(\delta+\langle u^*,a\rangle).$ Then for any bandit algorithm $\cB,$ there exists a $u^*\in\{-\Delta,\Delta \}^d $ such that the expected regret of $\cB$ over $K$ steps on bandit $\cL_{u^*}$ is lower bounded as follows:
    \begin{align}
        \bE_{u^*} \text{Regret}(K) \geq \frac{d\sqrt{K\delta}}{8\sqrt{2}}
    \end{align}
\end{lemma}
\noindent Then using the Lemmas above we can have:
\begin{align}
    \bE_u \left[\sum_{k=1}^K[V_1^*(x_1) -V^{\pi^k}_1(x_1)]  \right] \geq &\frac{H}{10} \sum_{h=1}^{H/2}\bE_u\left[\left(\max_{a\in\cA}\langle u_h,a\rangle - \langle u_h,\bar{a}_h^{\pi^k}\rangle \right) \right] \nonumber\\
    \geq &\frac{H^2}{20}\frac{d\sqrt{K\delta} }{8\sqrt{2}},
\end{align}
and
\begin{align}
  \bE_u \left[ \sum_{k=1}^K   \sum_{h=1}^{H}g_h(s_h,a_h)_+ \right] \geq & \sum_{h=1}^{H/2}(1-\frac{4}{3H})\bE\left[\bI_{\{\bar{a}_h^\pi \neq  \max_{a\in\cA}\langle u_h,a\rangle  \}}\right] \nonumber \\
  \geq & \frac{H}{2}(1-\frac{4}{3H})\frac{\sqrt{K\delta}}{8\sqrt{2}},
\end{align}
where the violation doesn't have the dependence on $d$ because by assumption on the cons function that only $1$ counts for the violation whenever the action is not the same as the optimal action. The result follows by plugging in $\delta=1/H.$

\section{More Experimental Discussions} \label{ap:sim}
\subsection{Heat Map}
To better illustrate the exploration strategies between LSVI-AE and the baseline LSVI-Primal, we present the heat map in Figure \ref{fig:heat}, where a darker grid represents fewer visitations. We can observe that our algorithm quickly restricts unsafe actions and finds the optimal solution after $500$ episodes.

\begin{figure}[!ht]
    \centering
    \includegraphics[scale=0.3]{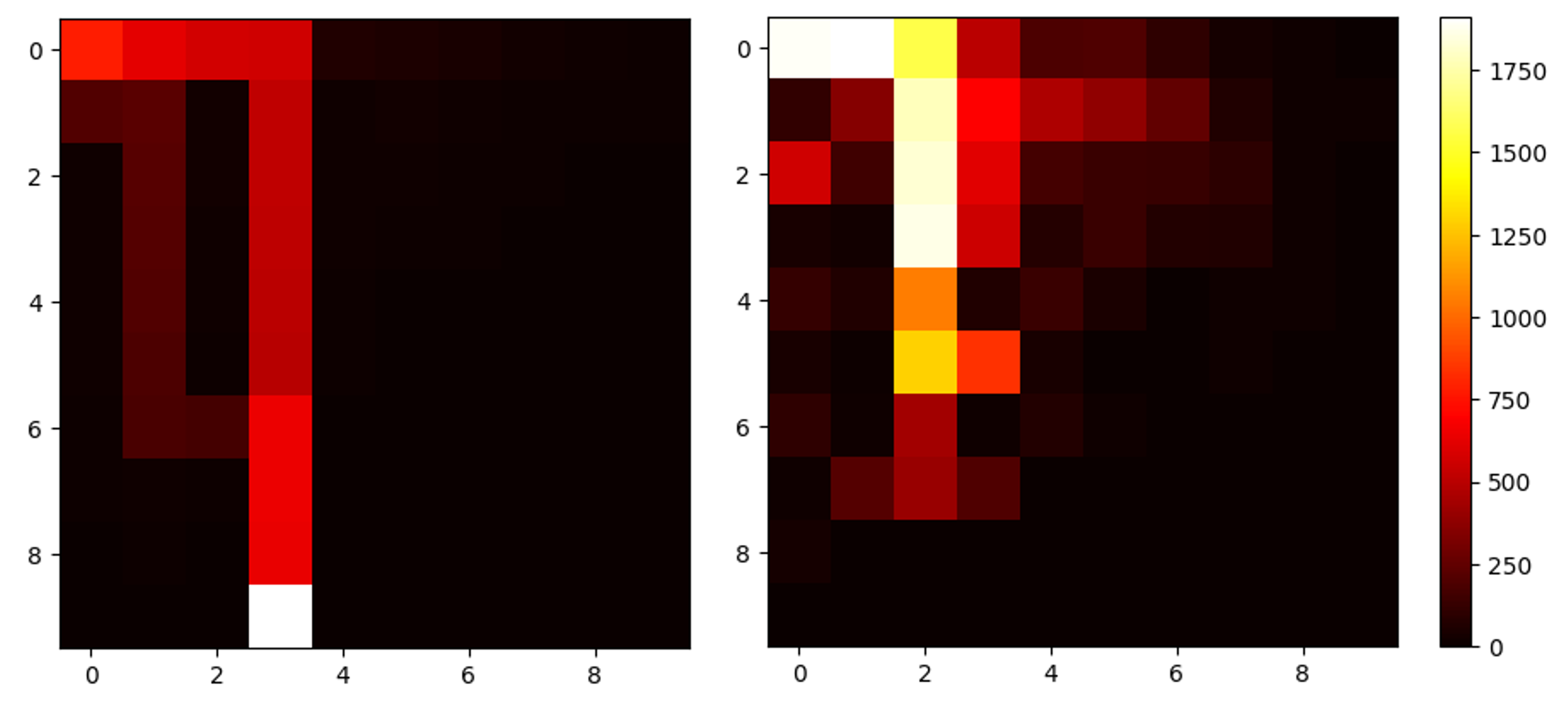}
    \caption{Heat Map after $500$ episodes}
    \label{fig:heat}
\end{figure}

\end{document}